\documentclass{article}

\PassOptionsToPackage{numbers, sort, compress}{natbib}

\usepackage[main, final]{neurips_2026}
\usepackage[pagebackref,breaklinks,colorlinks]{hyperref}
\usepackage[utf8]{inputenc} 
\usepackage[T1]{fontenc}    
\usepackage{url}            
\usepackage{amsfonts}       
\usepackage{nicefrac}       
\usepackage{microtype}      
\usepackage{xcolor}         
\usepackage{caption}
\usepackage{subcaption}
\usepackage{amsmath,amssymb,amsthm,mathtools,bm}
\newtheorem{theorem}{Theorem}[section]
\newtheorem*{theorem*}{Theorem} 
\newtheorem{lemma}{Lemma}[section]
\newtheorem{assumption}{Assumption}[section]
\newtheorem{definition}{Definition}[section]

\usepackage{amsmath,amsfonts,bm}

\def\eqref#1{equation~\ref{#1}}

\def\1{\bm{1}}

\def\mA{{\bm{A}}}

\def\mS{{\bm{S}}}

\def\mU{{\bm{U}}}
\def\mV{{\bm{V}}}

\DeclareMathAlphabet{\mathsfit}{\encodingdefault}{\sfdefault}{m}{sl}
\SetMathAlphabet{\mathsfit}{bold}{\encodingdefault}{\sfdefault}{bx}{n}

\usepackage{color}
\usepackage{marvosym}
\usepackage{enumitem}
\usepackage{colortbl}
\usepackage{arydshln}
\usepackage{adjustbox}
\usepackage{multirow}
\usepackage{makecell}
\usepackage{wrapfig}
\usepackage{booktabs}   
\usepackage{multirow}   
\usepackage{graphicx}   

\usepackage{tikz}

\newcommand{\bbR}{\mathbb{R}}
\newcommand{\cF}{\mathcal{F}}
\newcommand{\cQ}{\mathcal{Q}}
\newcommand{\cV}{\mathcal{V}}
\newcommand{\cT}{\mathcal{T}}
\newcommand{\bQ}{\mathbf{Q}}
\newcommand{\bV}{\mathbf{V}}
\newcommand{\bX}{\mathbf{X}}
\newcommand{\bH}{\mathbf{H}}
\newcommand{\bh}{\mathbf{h}}
\newcommand{\bU}{\mathbf{U}}
\newcommand{\bZ}{\mathbf{Z}}
\newcommand{\bq}{\mathbf{q}}
\newcommand{\bv}{\mathbf{v}}
\newcommand{\bz}{\mathbf{z}}
\newcommand{\ba}{\mathbf{a}}
\newcommand{\bW}{\mathbf{W}}
\newcommand{\bE}{\mathbf{E}}

\definecolor{rred}{RGB}{245, 152, 153}
\definecolor{oorange}{RGB}{253, 205, 154}
\definecolor{yyellow}{RGB}{255, 255, 153}
\definecolor{lightorange}{RGB}{251, 229, 214}

\title{Visual Enhanced Depth Scaling for Multimodal Latent Reasoning}

\author{
Yudong~Han\textsuperscript{1,2,3}\footnotemark[1], ~Yong Wang\textsuperscript{2}\footnotemark[1], ~Zaiquan Yang$^{4}$, ~Zhen Qu$^{5}$, ~Liyuan Pan\textsuperscript{1}\footnotemark[2], ~Xiangxiang Chu$^{2}$ \\
$^{1}$Computer Science, Beijing Institute of Technology, $^{2}$DreamX Team, \\
$^{3}$State Key Laboratory of General Artificial Intelligence, BIGAI, Beijing, China, \\
$^{4}$City University of Hong Kong, \\
$^{5}$Institute of Automation, Chinese Academy of Sciences, \\
\tt Code:~\href{https://github.com/Simon98-AI/Vedas}{https://github.com/Simon98-AI/Vedas}
}

\begin{document}
\maketitle

\begin{abstract}
Multimodal latent reasoning has emerged as a promising paradigm that replaces explicit Chain-of-Thought (CoT) decoding with implicit feature propagation, simultaneously enhancing representation informativeness and reducing inference latency. By analyzing token-level gradient dynamics during latent training, we reveal two critical observations: (1) visual tokens exhibit significantly smaller gradient norms than their textual counterparts due to inherent language bias, resulting in systematic visual under-optimization; and (2) semantically simple tokens converge rapidly, whereas complex tokens exhibit persistent gradient instability constrained by fixed architectural depths. To address these limitations, we propose a visual replay module and routing depth scaling to collaboratively enhance visual perception and refine complicated latents for deeper contextual reasoning. The former module leverages causal self-attention to estimate token saliency, reinforcing fine-grained grounding through spatially-coherent constraints. Complementarily, the latter mechanism adaptively allocates additional reasoning steps to complex tokens, enabling deeper contextual refinement. Guided by a curriculum strategy that progressively internalizes explicit CoT into compact latent representations, our framework achieves state-of-the-art performance across diverse benchmarks while delivering substantial inference speedups over explicit CoT baselines.

{
\renewcommand{\thefootnote}{\fnsymbol{footnote}}
\footnotetext[1]{\ Equal contribution. Work done when Yudong's internship at DreamX Team.}
\footnotetext[2]{\ Corresponding author.}
}

\end{abstract}

\section{Introduction}
\label{sec:intro}
Over the past few years, large language models (LLMs) have achieved remarkable progress in complex reasoning~\cite{DBLP:conf/cvpr/QuTBWQZW26, dream-rsi, yang2024language, dynfocus, deepseekr1}, demonstrating substantial capabilities across diverse applications such as mathematics~\cite{DBLP:conf/acl/Yan0HYHW25, DBLP:conf/iclr/LuoFSXZWYHT26, DBLP:journals/corr/abs-2605-09879}, code generation~\cite{DBLP:conf/nips/YangJWLYNP24, DBLP:conf/icml/WangCY0L0J24}, and AI for science~\cite{DBLP:conf/aaai/ZhangLHXCW26, DBLP:conf/cvpr/ZhangLLSWH26, DBLP:conf/nips/ZhangHCW25}, propelled by scaling laws in data volume~\cite{bai2025qwen3vltechnicalreport, kwaikeyeteam2026kwaikeyevl20technicalreport} and model capacity~\cite{DBLP:journals/corr/abs-2001-08361, Shukor2025ScalingLF}. Advanced techniques such as Chain-of-Thought (CoT) prompting~\cite{zhang2023multimodal, shao2024visual} and reinforcement learning (RL)~\cite{deepseekr1, Yu2025DAPOAO, Gao2025SoftAP, DBLP:conf/iclr/ChuHZWW26} for trajectory optimization have proven highly effective in text-only domains. Extending these capabilities to the multimodal realm has thus become a pivotal research direction. Current approaches primarily follow three paradigms. First, \textit{Text-based Reasoning}~\cite{zhang2023multimodal, MitraCCoT, lei2024scaffolding} generates explicit multi-step textual chains before producing an answer. However, these methods typically rely on static visual inputs, and recent studies~\cite{lookback2025, DBLP:journals/corr/abs-2502-17425} indicate that visual grounding deteriorates significantly over extended reasoning chains. Second, \textit{Tool-augmented Reasoning} manipulates visual inputs through external operations (e.g., zooming or region enhancement) and injects intermediate visual hints into the reasoning trace. While powerful, these approaches are prone to redundant or invalid tool invocations, introducing noise that degrades performance and substantially increases inference latency. Recently, emerging researches position \textit{Latent Reasoning} as a viable direction for optimizing multimodal reasoning. Unlike traditional methods that rely on explicit textual chains, latent reasoning encodes intermediate reasoning steps into compact continuous vectors. This paradigm offers compelling advantages, including higher inference efficiency, reduced annotation overhead, and the ability to learn dense, high-fidelity multimodal representations~\cite{zhang2023multimodal, DBLP:journals/corr/abs-2505-16552, tang2026barybind, li2025geometric, tang2026learning, li2026point}.

\begin{figure*}[]  
    \centering
    \begin{subfigure}[b]{0.46\linewidth}
        \centering
        \includegraphics[width=\linewidth, height=3.5cm]{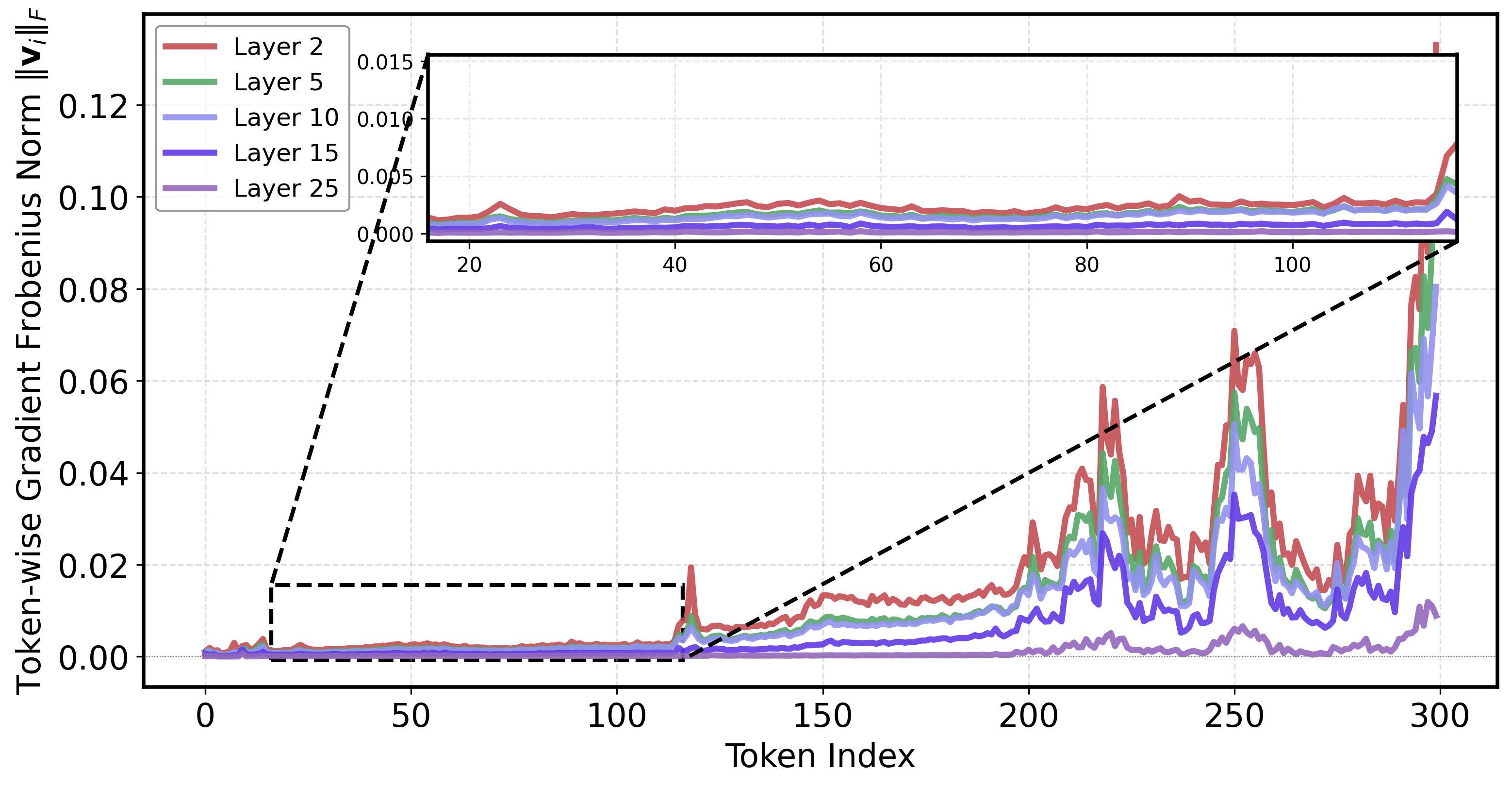}
        \caption{Token Frobenius Norm in Different Epochs}
        \label{fig:v_vs_t}
    \end{subfigure}
    \hfill
    \begin{subfigure}[b]{0.26\linewidth}  
        \centering
        \includegraphics[width=\linewidth, height=3.5cm]{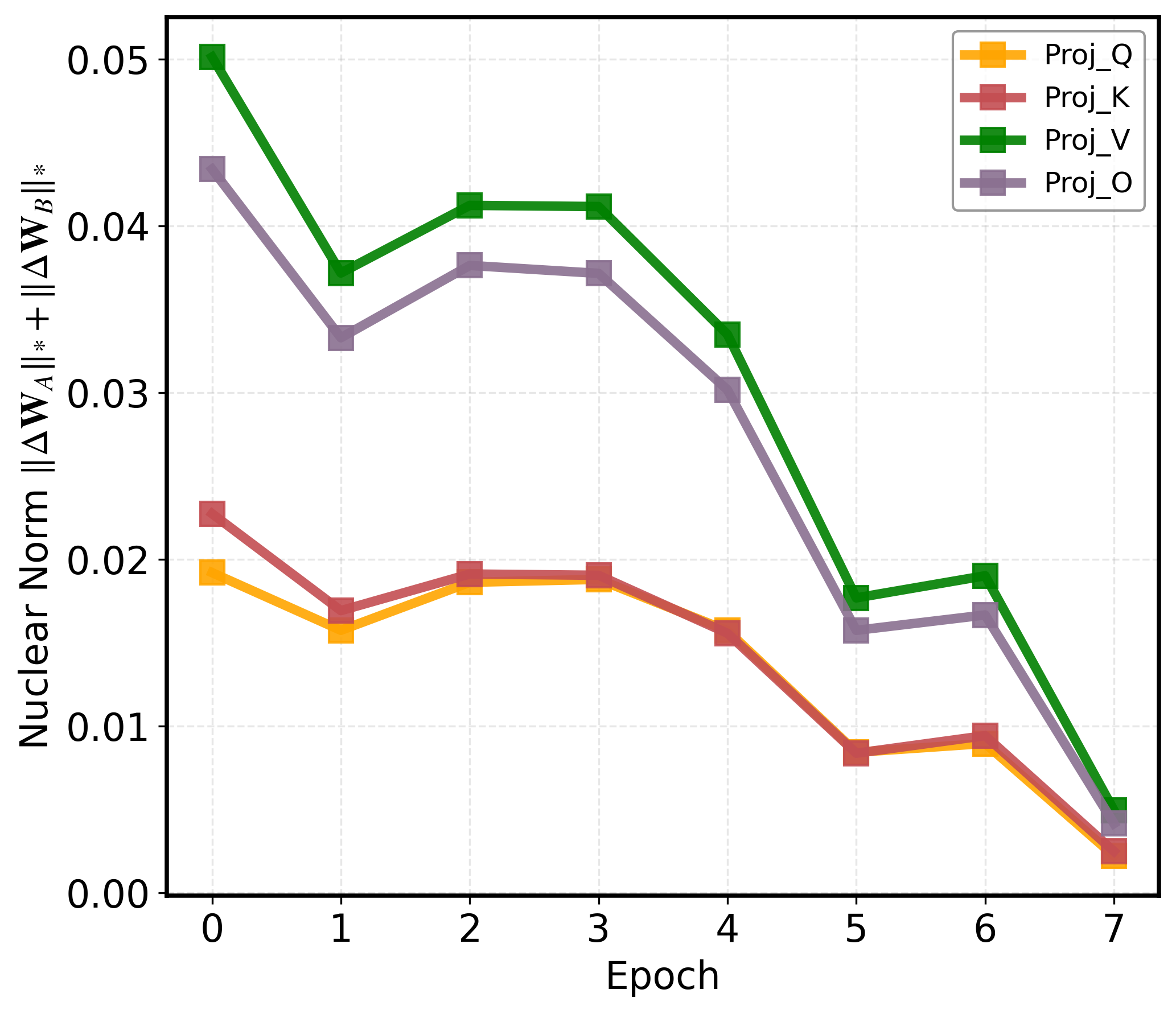}
        \caption{Easy Samples}
        \label{fig:easy}
    \end{subfigure}
    \hfill            
    \begin{subfigure}[b]{0.26\linewidth}
        \centering
        \includegraphics[width=\linewidth, height=3.5cm]{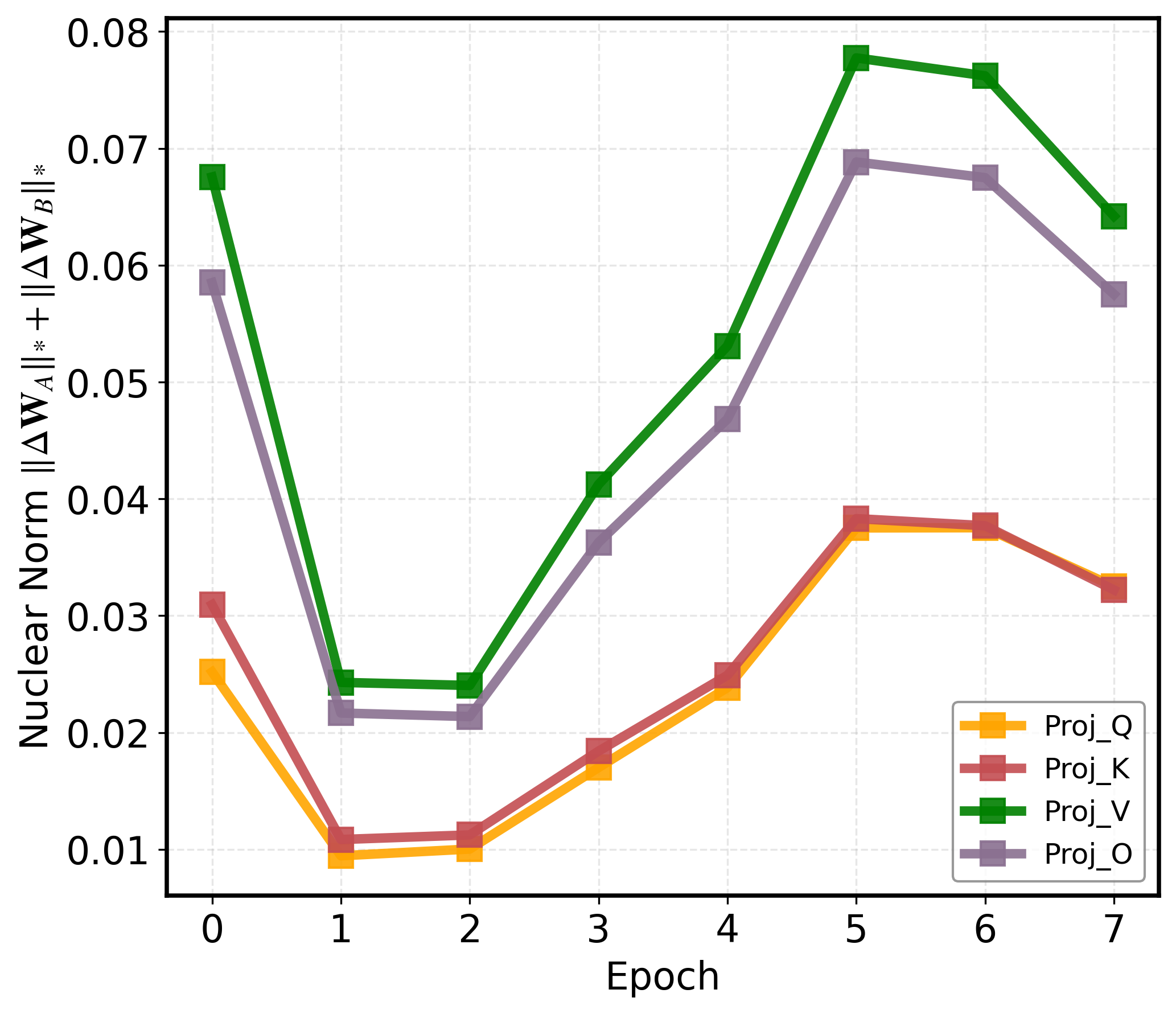}
        \caption{Hard Samples}
        \label{fig:hard}
    \end{subfigure}%
    \caption{Panel (a) depicts the token-wise Frobenius norm of gradients across different layers throughout the entire training process, revealing a clear gradient imbalance between visual tokens and text tokens/latent token. Panels (b) and (c) further reveal distinct evolution patterns of the Nuclear norm in the QKVO projection matrix across layers and training epochs: gradients for easy samples decay rapidly and converge smoothly, whereas hard samples maintain elevated gradient norms with persistent oscillations. Here, $\mathbf{W}_{A}$ and $\mathbf{W}_{B}$ denote the low-rank matrix in LoRA~\cite{hu2022lora}.
    }
    \label{fig:observation}
    \vspace{-1em}
\end{figure*}

To gain deeper insights into the optimization dynamics of latent reasoning, we conducted a systematic analysis of gradient flows and parametric evolution during training. This investigation reveals two meaningful observations: \textit{(1) Visual-Text Optimization Disparity:} Recent studies~\cite{DBLP:journals/corr/abs-2507-03019, DBLP:conf/iclr/0002KCI25, DBLP:conf/acl/WangXLHWWZC26} have highlighted the phenomenon of visual attention attenuation in MLLMs as the explicit Chain-of-Thought (CoT) reasoning chain extends. We discover that a similar degradation existing in latent reasoning. As depicted in Fig.~\ref{fig:observation}\subref{fig:v_vs_t}, the gradient updates for visual tokens (regions marked by dashed lines) are remarkably small, while the overall optimization is predominantly governed by gradients from text tokens and latent token. This suggests that the model relies primarily on textual information and latent representations to generate final answers during training, revealing a pronounced gradient imbalance phenomenon.
\textit{(2) Fixed-Depth Optimization Dilemma:} Beyond modality imbalance~\cite{DBLP:journals/inffus/ZhangDHCWH26, DBLP:conf/iccv/WeiLL25}, we observe an architecture bottleneck in how models handle samples with different complexity. Partitioning the training data into \textit{easy} (consistently correct) and \textit{hard} (persistently incorrect) subsets based on early-stage validation accuracy, we tracked the gradient nuclear norms across the QKV and output projection matrices O (Fig.~\ref{fig:observation}\subref{fig:easy} and Fig.~\ref{fig:observation}\subref{fig:hard}). The trends reveal a stark divergence: easy samples exhibit smooth gradient decay, indicating stable convergence into favorable loss basins. In contrast, hard samples maintain persistently high gradient volatility even in late training epochs. This phenomenon underscores the necessity of iterative refinement for complex reasoning, which is crucial for parsing compositional patterns in complex contexts while mitigating inherent visual ambiguities in images. Fixed-depth architectures, however, lack the flexibility to adapt to varying token complexities, thereby trapping hard examples in oscillatory optimization trends.

Motivated by these observations, we propose a unified framework (Vedas) that strengthens fine-grained visual engagement and achieves token-wise depth scaling strategy to enable more precise and comprehensive contextual reasoning. Firstly, to mitigate visual optimization instability, we design a visual replay module, which dynamically replays the focused visual clues interleaved with thinking latents. This mechanism iteratively exposes and propagates key visual context across reasoning steps, fostering step-wise alignment with the target answer. Complementing this, we apply self-distillation supervision to enforce spatial coherence and preserve fine-grained visual details within the visual latents. Secondly, to address the fixed-depth optimization bottleneck, we propose a per-layer token router that dynamically allocates additional reasoning steps based on token complexity or information density. This design enables high-difficulty tokens to engage in prolonged contextual reasoning by reusing layer-wise knowledge, facilitating iterative representation refinement and adaptive prioritization of salient contextual cues.
Finally, in contrast to methods~\cite{hao2024training, codi} that rely on knowledge distillation for direct latent supervision, we employ a curriculum learning strategy that progressively introduces latent tokens into the training pipeline. Extensive experiments show that our method can be effortlessly combined with various widely-used MLLM backbones to further enhance reasoning performance while maintain the satisfactory inference latency. The main contributions can be summarized as follows:
\begin{itemize}
\item We systematically analyze token-level gradient dynamics during latent reasoning training, revealing two critical optimization bottlenecks: visual-text optimization disparity and fixed-depth optimization dilemma.
\item We present a unified curriculum-driven framework that progressively constructs interleaved latent representations. By integrating spatially-coherent visual constraints for fine-grained grounding and complexity-aware depth scaling, our approach enables robust and precise contextual reasoning.
\item Extensive experiments across twelve widely-used multimodal reasoning benchmarks demonstrate that our method achieves state-of-the-art performance while maintaining high inference efficiency.
\end{itemize}

\section{Related Work}
\label{sec:related_work}

\noindent \textbf{Explicit Multimodal Reasoning.} 
Multimodal reasoning enables model to reason over information from different modalities to solve complex tasks. There are many prior works~\cite{Yu2026TheLS, hu2022promptcap, ssa, Zheng_NeurIPS2023, MitraCCoT, DBLP:journals/corr/abs-2601-01984, MitraCCoT, kamcot, dynfocus, super, Yu2025VisMemLV, Liu2026MemaMA, yang2025unleashing, yang2024boosting, li2026scenescaffold} focusing extensively on enhancing reasoning capabilities. Earlier work rely on the CoT prompting to perform explicit thinking steps in the text space before generating the final answer. However, this paradigm generally lack sufficient visual grounding capability and leads to unsatisfactory misalignment and hallucination~\cite{huang-etal-2024-visual, DBLP:journals/corr/abs-2404-18930, tian2025core}.
To address these limitations, recent studies~\cite{Zheng_NeurIPS2023, hu2022promptcap, MitraCCoT} have explored converting visual information into textual formats prior to reasoning, leveraging external tools or specialized visual experts to generate descriptive representations that guide LLMs. For instance, \citet{hu2022promptcap} pioneered the integration of visual captions, extracting semantic content as text and concatenating it with input prompts to bolster reasoning capabilities. To further enhance fine-grained reasoning, subsequent works have focused on regional understanding, aiming to improve textual expressiveness by describing specific image regions. Others~\cite{DBLP:journals/corr/abs-2601-01984, MitraCCoT, kamcot} identified entities and their relationships within images, facilitating fine-grained reasoning through explicit modeling of inter-entity connections.

A line of concurrent works advocates using vision-text interleaved format during the rationale generation and reasoning process. The model draws auxiliary lines or marks based on original image to record thinking path, zoom or crop regions, or perform code editing, etc. Building on these paradigms, \citet{zhang2023multimodal} first proposed decoupling rationale generation from answer generation in the Vision-Text Reasoning field. Subsequently, \citet{Shao2024VisualCA} annotated key regions of the original image in intermediate steps, training models to focus on image regions relevant to the answer. While some works~\cite{gao2024interleaved, zhang2025chain} further extract key image regions progressively during reasoning, combining visual information with textual reasoning to generate the final answer. Moreover, new methods~\cite{visualsketchpad, liu2025Visual} emulated human thought by sketching images during reasoning, focusing on core concepts, structures, and relationships while ignoring redundant details. Other works~\cite{li2025imaginereasoningspacemultimodal, chern2025thinking, tian2025open} generated new images with auxiliary markers during reasoning, combining them with text to improve reasoning in complex scenarios. To fully shift reasoning from the linguistic domain to the visual modality, \citet{xu2025visualplanningletsthink} proposes to reasoning exclusively with dynamic generated images, demonstrating substantial performance gains in visual navigation tasks.

More recently, Vision-R1~\cite{DBLP:journals/corr/abs-2503-06749} and VL-Rethinker~\cite{DBLP:journals/corr/abs-2504-08837} leveraged Group Relative Policy Optimization~\cite{deepseekr1} (GRPO) to refine reasoning trajectories through rollout-based sampling and reward scoring. Complementing these policy-driven approaches, concurrent works further enhance reasoning capabilities via novel cognitive paradigms, including self-critiquing cycles~\cite{DBLP:journals/corr/abs-2405-06682, DBLP:conf/cvpr/CocchiMC0C25, dycpo}, iterative rethinking~\cite{lookback2025}, and on-policy distillation~\cite{DBLP:journals/corr/abs-2601-18734, Un-OPD}.

\noindent \textbf{Latent Reasoning.} 
Different from explicit reasoning in the discrete token space, latent reasoning refers to internal computation performed in a hidden space before answer generation. \citet{hao2024training} pioneers continuous latent space reasoning by feeding the last hidden states as input embeddings for the next step without generating intermediate discrete tokens, substantially reducing reasoning latency. However, subsequent studies have indicated that such paradigm may suffer from feature homogenization without explicit supervision on intermediate latent states. To address this limitation, a series of works attempt to enhance the quality of intermediate representations using diverse strategies. \citet{cheng2024compressed} introduced variable-length contemplation tokens for latent reasoning, mitigating quality degradation caused by fixed-length constraints. Similarly, \citet{shen2025codi} leveraged distillation tactic to align student and teacher hidden activations along with explicit supervision, thereby constraining latent reasoning paths. Beyond these efforts, \citet{DBLP:journals/corr/abs-2509-20317} adopted step-level supervision to further stabilize the reasoning space.

Recently, latent reasoning has been extended to Multimodal Large Language Models (MLLMs). Distinct from text-only LLMs, MLLMs necessitate the effective integration of visual features within the latent reasoning space. Several efforts \cite{yang2025machine, li2025latent, pham2025multimodal, liu2025reasoningminddynamicmultimodal,laser2026forest, mema, Yu2025VisMemLV} have been dedicated to injecting visual cues into the latent space to facilitate visual-grounded reasoning. For instance, \citet{laser2026forest} emphasize image details by constructing a structured cognitive hierarchy, albeit relying on annotation-intensive multimodal reasoning data. Similarly, \citet{liu2025reasoningminddynamicmultimodal} progressively select visual patches to inject into latent thinking tokens via confidence-guided policy gradient optimization. In this work, we systematically analyze gradient dynamics during latent reasoning training and reveal two critical bottlenecks towards token-wise optimization behavior.

\section{Method}
\label{sec:method}

\subsection{Preliminary: Implicit and Explicit Decoding}
Given a multimodal input comprising a question $\cQ$ and an image $\cV$, we tokenize them into a text embedding sequence $\bQ=\{\bq_i\}_{i=1}^{N_q}$ and visual features $\bV=\{\bv_i\}_{i=1}^{N_v}$ using a word embedding matrix $\bE\in\bbR^{|\mathcal W|\times d}$ and a pretrained visual encoder, respectively. We concatenate $\bX=[\bQ\|\bV]$ and use an autoregressive MLLM $\cF_\theta$ to obtain the prefilling hidden states $\bH^{(0)}\in\bbR^{P\times d}$, where $P=N_q+N_v$. During decoding, we fix $T_r$ implicit reasoning steps and $T_a$ explicit answer tokens. The generation process has two stages.

\noindent{\bf Implicit Reasoning Phase.}
At step $t$, the model produces a continuous latent vector $\bz^{(t)}$ conditioned on $\bX$ and previous latent states:
\begin{equation}
\bz^{(t)}=\cT\!\left(\cF_\theta\!\left(\bX \,\|\, \bZ_{<t}\right)\right),\quad t=1,\dots,T_r,
\end{equation}
where $\bZ_{<t}=\{\bz^{(1)},\dots,\bz^{(t-1)}\}$ and $\cT(\cdot)$ extracts the hidden state of the last token. The latent sequence $\bZ_{1:T_r}$ is then concatenated to the context for answer decoding.

\noindent{\bf Explicit Answer Decoding.}
The answer tokens $\ba_{1:T_a}$ are generated autoregressively. For $t$-th answer token $a_t$:
\begin{equation}
p_\theta(a_t \mid \bX,\bZ_{1:T_r},a_{<t})
=\mathrm{Softmax}\!\left(\bW_o\bh_t^{\text{dec}}\right),
\end{equation}
where $\bW_o\in\bbR^{|\mathcal W|\times d}$ is the output projection matrix and $\bh_t^{\text{dec}}$ is the decoder hidden state at step $t$.
The joint likelihood factorizes as
\begin{equation}
p_\theta(\ba_{1:T_a}\mid\bX,\bZ_{1:T_r})
=\prod_{t=1}^{T_a} p_\theta(a_t\mid\bX,\bZ_{1:T_r},a_{<t}).
\end{equation}
At each decoding step, the decoder input is $\bU_t=[\bX\|\bZ_{1:T_r}\|a_{<t}]$, with $\|$ denoting concatenation along the token dimension.

\begin{figure*}
    \centering
    \includegraphics[width=1.0\linewidth]{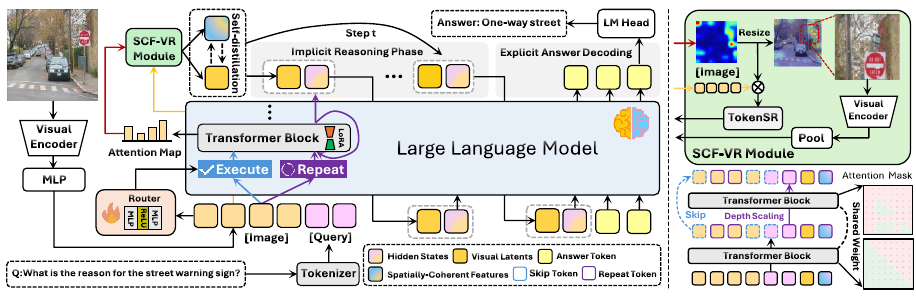}
    \caption{\noindent\textbf{Left Panel:} Schematic illustration of our framework. Input images are encoded into visual tokens by a pretrained visual encoder and projected into a text-centric semantic space aligned with the LLM, while questions are tokenized by the text tokenizer. Our method enhances a standard latent MLLM with two synergistic components during the implicit reasoning phase: Spatially-Coherent Finer Visual Replay (SCF-VR) and Routing Depth Scaling (RDS). The overall objective is jointly optimized via standard cross-entropy loss and self-distillation loss. \textbf{Right Panel:} Detailed architecture of the SCF-VR module and token-wise depth scaling mechanism. \textbf{TokenSR} denotes the token super-resolution module, and \textbf{Pool} is the average pooling.}

    \label{fig:framework}
\end{figure*}

\subsection{Spatially-Coherent Finer Visual Replay}
As mentioned earlier, we empirically reveal the gradient disparities between visual and textual tokens throughout the learning dynamics. Specifically, visual tokens consistently exhibit substantially smaller gradient Frobenius norms compared to textual tokens and latent token, indicating that visual representations remain under-optimized despite their critical role in multimodal reasoning, which aligns with the conclusions drawn in prior studies~\cite{DBLP:conf/cvpr/PengWD0H22, DBLP:journals/corr/abs-2507-10203}. Motivated by these insights, we introduce the visual replay module to reinforce the engagement of visual cues via salient region detection, while enhancing fine-grained spatially-coherent perception capabilities via self-distillation supervision at each reasoning step.

\noindent{\bf Attention-Guided Region Focus.} 
Several works~\cite{Li_2025_ICCV, DBLP:conf/naacl/ZhangQSYYXWGTY25, dycpo, DBLP:journals/corr/abs-2505-14231} have demonstrated that LLMs exhibit fundamental visual grounding capabilities. To identify visually salient regions, we aggregate attention weights across all transformer layers and attention heads to obtain a consolidated spatial focus map. Specifically, given the $l$-th layer and the $h$-th attention head, we compute the mean attention map $\bar{\mathbf{A}}^{(t)}$ at reasoning step $t$:
\begin{equation}
\bar{\mathbf{A}}^{(t)} = \frac{1}{L \cdot H} \sum_{l=1}^{L} \sum_{h=1}^{H} \mathbf{A}^{(l,h,t)},
\label{eq:agg_attention}
\end{equation}
where $\mathbf{A}^{(l,h,t)} \in \mathbb{R}^{P^{(t)} \times P^{(t)}}$ denotes the attention matrix for layer $l$ and head $h$ at iteration $t$ ($1 \leq t \leq T_r$), with $P^{(t)}$ representing the number of input tokens at iteration $t$. Here, $L$ and $H$ represent the total number of layers and heads, respectively. To obtain token-level attention scores, we extract the attention distribution from the most recently generated token to all preceding tokens via column-wise summation, i.e., $\mathbf{a}_{all}^{(t)} = \mathrm{colsum}(\bar{\mathbf{A}}^{(t)}) \in \mathbb{R}^{P^{(t)}}$. Subsequently, we extract only the visual token attention scores from $\mathbf{a}_{all}^{(t)}$ using the image mask, denoted as $\mathbf{a}^{(t)} \in \mathbb{R}^{N_v}$, where $N_v$ is the number of visual tokens. This visual attention vector effectively captures which visual tokens are most relevant to the current reasoning context.

As visual focus evolves across reasoning steps, we iteratively select the top-$K$ attended visual tokens $\{ \mathbf{v}_{i}^{(t)}\}_{i=1}^{K}$ as visual latents, which are integrated with hidden states $\mathbf{z}_{t}$. To prevent redundant re-selection and promote diverse exploration, we maintain a visited token set $\mathcal{V}^{(t)}_{\text{visited}}$, ensuring comprehensive visual coverage:
\begin{equation}
\mathcal{I}^{(t)} = \operatorname{TopK}\left(
    \{ a^{(t)}_{i} \mid i \notin  \mathcal{V}^{(t)}_{\text{visited}} \}\right),
\label{eq:topk_selection}
\end{equation}
where $\mathcal{I}^{(t)}$ represents the indices of the $K$ visual tokens with the highest attention scores at step $t$, and $\mathcal{V}^{(t)}_{\text{visited}}$ is updated after each selection. The original embeddings of the selected tokens $\mathbf{V}_{\mathcal{I}^{(t)}} = \{\mathbf{v}_{i} \mid i \in \mathcal{I}^{(t)}\}$ are further weighted by normalized attention scores:
\begin{equation}
\mathbf{B}^{(t)} = \mathrm{Diag}(\mathrm{Softmax}(\{a^{(t)}_{i} | i \in \mathcal{I}^{(t)} \})) \mathbf{V}_{\mathcal{I}^{(t)}},
\label{eq:attention_weight}
\end{equation}
where $\mathrm{Diag}(\cdot)$ constructs a diagonal matrix from a vector.

\noindent{\bf Spatially-Coherent Regularization.}
Although leveraging learned attention within Transformers provides explainable visual locations, it often suffers from limited spatial continuity due to the scattered nature of selected tokens and introduces noise associated with the attention sink phenomenon~\cite{DBLP:conf/iclr/XiaoTCHL24}. To mitigate these issues and enhance fine-grained perception~\cite{DBLP:conf/eccv/MaYJHHWCC26} without external annotations, we introduce self-distillation supervision. This mechanism involves cropping the visual regions exhibiting spatial coherence, re-encoding them, and supervising the visual latents with these high-fidelity features.

Specifically, we first search for a $W \times W$ sub-grid patch that maximizes the density of attended visual tokens. Formally, we find the optimal top-left corner 
$(r^*, c^*)$ within the valid grid bounds:
\begin{equation}
(r^*, c^*) = \underset{0 \leq r, c \leq G-W}{\operatorname{argmax}} \; \mathcal{N}(r, c),
\end{equation}
where the density function $\mathcal{N}(r, c)$ counts the visited tokens falling within the window $\mathcal{N}(r, c) = \sum_{i \in \mathcal{I}^{(t)}} \mathbb{I}\big[r \leq r_i < r+W, c \leq c_i < c+W\big].$ Here, $(r_i, c_i)$ denotes the row-column position of token $i$ on the $G \times G$ grid, and $\mathbb{I}[\cdot]$ is the indicator function.
This greedy selection ensures the cropped region captures a coherent visual context rather than scattered details. Then, the selected window is projected to pixel coordinates in the original image via mathematical transformation, cropped, 
and resized to the standard encoder input resolution using bilinear interpolation. We then re-encode this refined patch through the same visual encoder to obtain fine-grained representations $\{\mathbf{f}_{i}\}_{i=1}^{N_v'}$. 
A global pooling operation $\mathrm{Pool}(\cdot)$ is applied to yield a robust reference token $\mathbf{u}^{\text{ref}}$.
Finally, we align the coarse-grained global token $\mathbf{b}^{(t)} = \mathrm{Pool}(\mathbf{B}^{(t)})$ with this high-fidelity reference using a lightweight 
token super-resolution module $\mathcal{F}_{\text{SR}}: \mathbb{R}^D \to \mathbb{R}^D$. 
We minimize the reconstruction error as follows:
\begin{equation}
\mathcal{L}_{\text{recon}}^{(t)} = \left\| \mathcal{F}_{\text{SR}}\left(\mathbf{b}^{(t)}\right) - \mathbf{u}^{\text{ref}} \right\|_2^2.
\label{eq:reconstruction_loss}
\end{equation}
This supervisory signal enables the model to prioritize spatially coherent and semantically intact visual contexts during latent generation through self-distillation.
\subsection{Routing Depth Scaling}
\label{method:routing}

While visual replay mechanisms significantly enhance fine-grained grounding by incorporating spatially-coherent visual latents, existing methods typically allocate uniform computational budgets across prefilling tokens and newly generated latents during contextual refinement. Intuitively, however, tokens and latents contribute heterogeneously to the final reasoning objective, rendering static budget allocation suboptimal. We hypothesize that a token's processing requirement is intrinsically linked to its relative \textit{difficulty}, which can be quantified via optimization dynamics.

\noindent\textbf{A Theoretical Perspective: Rank Dynamics as a Proxy for \textit{Difficulty}.}
To formalize \textit{difficulty}, we analyze the optimization dynamics of QKVO weight matrices, positing that their evolution reflects sample complexity. For simple samples, consistent gradient updates drive weights toward a compact, low-rank structure where principal singular values dominate and noise is suppressed. Conversely, difficult samples involve conflicting contexts that force representations to constantly adjust between competing patterns. This instability manifests as \textit{resurrected} suppressed singular values or reshaped principal distributions, causing persistent fluctuation in the nuclear norm of matrix. We quantify these dynamics via the nuclear norm, defined as follows:

\begin{definition}[Nuclear Norm of a Matrix]
\itshape
\label{def:nuclear_norm_matrix}
Let $\mA \in \mathbb{R}^{m \times n}$ be a weight matrix in transformer, the Singular Value Decomposition (SVD) of $\mA$ is given by $\mA = \mU \mS \mV^\top$, where $\mU \in \mathbb{R}^{m \times m}$ and $\mV \in \mathbb{R}^{n \times n}$ are orthogonal matrices, and $\mS \in \mathbb{R}^{m \times n}$ is a diagonal matrix containing the non-negative singular values $\sigma_1 \ge \sigma_2 \ge \dots \ge \sigma_r > 0$, with $r = \mathrm{rank}(\mA)$. The nuclear norm of $\mA$, denoted as $\|\mA\|_*$, is defined as the sum of its singular values:
\[
  \|\mA\|_* = \sum_{i=1}^{\min(m,n)} \sigma_i(\mA) = \mathrm{Tr}\left(\sqrt{\mA^\top \mA}\right).
\]
\end{definition}
Based on above defined \textit{difficulty} indicator, we reveal heterogeneous optimization complexities across tokens during latent training, termed the \textit{fixed-depth optimization dilemma}. To address this without modifying the pretrained VLM architecture while effectively leveraging its inherent knowledge, we introduce a lightweight router that dynamically allocates additional reasoning steps exclusively to critical tokens, enabling adaptive reasoning depths.

\noindent\textbf{Router Network.} 
Let the input token sequence at the $t$-th reasoning step be denoted as $\mathbf{U}^{(t)} = [\mathbf{X} \| \mathbf{B}^{(1)} \| \mathbf{z}^{(1)} \| \cdots \| \mathbf{B}^{(t)} \| \mathbf{z}^{(t)}] \in \mathbb{R}^{P^{(t)} \times d}$, where $P^{(t)}$ represents the total sequence length and $d$ is the hidden dimension. To facilitate the illustration of our depth scaling mechanism, we decompose the forward pass into layer-wise computations. Specifically, within the $l$-th transformer layer, we first compute the intermediate feature $\mathbf{H}^{(l, t)} = f(\mathbf{U}^{(t)}) \in \mathbb{R}^{P^{(t)} \times d}$, where $f(\cdot)$ denotes the transformation of the $l$-th transformer layer in LLM. Subsequently, a lightweight router network computes a scalar importance score for each token based on its corresponding hidden representation:
\begin{equation}
    \mathbf{s}^{(l, t)} = \mathcal{F}_{\text{router}}(\mathbf{H}^{(l, t)}) \in \mathbb{R}^{P^{(t)}},
\end{equation}
where the $i$-th element $s^{(l, t)}_i$ quantifies the importance of the $i$-th token in the $l$-th layer, and $\mathcal{F}^{(l)}_{\text{router}}(\cdot)$ denotes the independent router network in the $l$-th layer. We define $T_{\alpha}(\mathbf{s}^{(l, t)})$ as the index set of the top-$\alpha$ tokens with the highest scores, where $\alpha$ serves as a predefined hyper-parameter. In practice, $\alpha$ can be formulated as a layer-dependent function $\alpha(l)$ to enable adaptive resource allocation at different layers.

\noindent\textbf{Depth Scaling Computation.} Obtaining the router weight, we select the top-$\alpha$ tokens for depth scaling, i.e., repeat iteration in one transformer block. For a given depth scaling step $d \in \{1, \ldots, D\}$, the token-wise representation update rule within a transformer layer can be generally formulated as follows:
\begin{equation}
    \mathbf{h}_{i}^{(d)} = 
    \begin{cases}  s^{(d)}_{i} \cdot f_{\in T_{\alpha}(\mathbf{s})}(\mathbf{h}^{(d-1)}_{i}, \mathbf{m}^{(d-1)}), & \text{if } i \in T_{\alpha}(\mathbf{s}^{(d)}), \\
        \mathbf{h}_{i}^{(d-1)}, & \text{if } i \notin T_{\alpha}(\mathbf{s}^{(d)}).
    \end{cases}
    \label{eq:selective_update}
\end{equation}
For notational brevity, we omit the layer index $l$ and reasoning step $t$, denoting refinement depth by $d$ (e.g., $\mathbf{s}^{(l,t)} \to \mathbf{s}^{(d)}$). Let $\mathbf{h}^{(0)}_{i}$ denote the initial hidden state of token $i$, and $\mathbf{m}^{(d)}$ the attention mask at step $d$. A router network dynamically selects a critical token subset $T_{\alpha}(\mathbf{s}^{(d)})$ based on importance scores, acting as a binary gate. Consequently, the transformation $f(\cdot)$ is applied exclusively to tokens $i \in T_{\alpha}(\mathbf{s}^{(d)})$ for $d$ refinement steps, while non-critical tokens retain their prior states to ensure efficiency. Finally, we aggregate the refined representations with step-aware positional encoding$~\mathbf{e}^{(d)}$ to further form the final hidden state:

\begin{equation}
    \mathbf{H}^{(D)} = \mathbf{H}^{(0)} + \sum_{d=1}^{D} \left( \mathbf{H}^{(d)} \odot \mathbf{e}^{(d)} \right).
    \label{eq:aggregation}
\end{equation}
This adaptive strategy allocates deeper reasoning pathways to critical tokens, enabling the model to capture complex visual contexts and facilitate more profound reasoning capabilities. We further provide a theoretical analysis demonstrating that, under mild assumptions, our design achieves a tighter generalization bound (see \textbf{Appendix}).
\begin{theorem}[Generalization Analysis of Token-wise Depth Scaling]
\itshape
\label{thm:main}
Let $\mathcal{T}_{\text{std}}$ denote the standard $L$-layer Transformer with total parameter count $W_{\text{std}} = \Theta(Ld^2)$. Let $\mathcal{T}_{\text{topk}}$ denote the token-wise depth scaling variant with Top-K selection, having total parameter count $W_{\text{topk}} = W_{\text{std}} + Md = W_{\text{std}}(1 + O(1/d))$, where $M = \lceil L/2 \rceil$ is the number of modified layers. Under \textit{Assumptions~\ref{ass:weight-bounded}--\ref{ass:regularity}}, with probability at least $1-\delta$, the population risk satisfies:
\begin{equation}
\mathcal{L}_{\mathcal{D}}(f_{\text{topk}}) \leq \mathcal{L}_{\mathcal{D}}(f_{\text{std}}) - \frac{\delta_0 K L}{4n} \cdot \eta_{\text{eff}} + \tilde{\mathcal{O}}\left(\sqrt{\frac{dL\log(nL/\delta)}{N}}\right) - \Omega\left(\frac{K^2L}{n^2}\sqrt{\frac{W_{\text{std}}}{N}}\right),
\label{eq:bound}
\end{equation}
where $\delta_0 > 0$ is the minimal improvement from single-layer iteration, $K = \rho n$ is the selection number, $\eta_{\text{eff}}$ is the effective selection rate, $d$ is the hidden dimension, and $N$ is the number of training samples.
\end{theorem}

\subsection{Training Procedure}
\label{subsec:training_objective}
\noindent\textbf{Curriculum Latent Training.} 
To reduce annotation overhead and avoid human priors limiting learning, we do not directly supervise intermediate latent representations. Instead, we introduce a curriculum that links implicit latents with explicit reasoning chains. 
We first train the model with standard Chain-of-Thought (CoT) supervision, generating all reasoning steps explicitly to build baseline reasoning ability. Then, over training, we gradually replace one explicit step at a time with an informative $\langle \text{latent} \rangle$ token. This curriculum enables each latent to progressively ground relevant contextual cues, internalizing explicit reasoning into compact latent representations.

\noindent\textbf{Training Objective.} 
The overall training objective combines the standard language modeling loss with the self-distillation loss in the SCF-VR module,
\begin{equation}
\mathcal{L}_{\text{total}} = 
\mathcal{L}_{\text{CE}} + \lambda \frac{1}{T_{r}} \sum_{t=1}^{T_{r}} \mathcal{L}_{\text{recon}}^{(t)}, ~~~\mathcal{L}_{\text{CE}} = -\sum_{t=1}^{T_{r} + T_{a}}  \log p_{\theta}\left( a_{t} | \mathbf{U}^{(t)}, a_{<t} \right),
\label{eq:total_loss}
\end{equation}
where $\mathcal{L}_{\text{CE}}$ is the standard cross-entropy loss over the language modeling task,
where $\lambda$ is a hyperparameter balancing the primary task and the auxiliary reconstruction objective. During inference, only the LLM component is active, and the visual play module is disabled, ensuring no additional computational overhead at test time.

\begin{table*}[t]
\centering
\caption{Comparison of various multimodal reasoning baselines across three benchmarks. We selected three datasets featuring detailed reasoning chains for training and evaluate on their test split, respectively. Three metrics are reported: Accuracy (\%), Average number of Autoregressive Steps (\# AR steps), and Average Generation Time (AVG. Time). $^\dagger$ denotes the reimplementation for the methods with the same configuration with ours. No-CoT notes that directly predicts answers without generating intermediate steps.} 
\small
\renewcommand{\arraystretch}{1.05}
\resizebox{1.0\linewidth}{!}{%
\setlength{\tabcolsep}{8pt}
\begin{tabular}{ll@{\hspace{6pt}}ccc@{\hspace{6pt}}ccc@{\hspace{6pt}}ccc}
\toprule[1.5pt]
\multicolumn{2}{c}{} &
\multicolumn{3}{c}{\textbf{M$^3$CoT}} &
\multicolumn{3}{c}{\textbf{ScienceQA}} &
\multicolumn{3}{c}{\textbf{GQA}} \\
\cmidrule(lr){3-5} \cmidrule(lr){6-8} \cmidrule(lr){9-11}
\textbf{Method} & \textbf{Model} &
Acc.(\%) $\uparrow$ & \# AR Steps $\downarrow$ & Avg. Time $\downarrow$ &
Acc.(\%) $\uparrow$ & \# AR Steps $\downarrow$ & Avg. Time $\downarrow$ &
Acc.(\%) $\uparrow$ & \# AR Steps $\downarrow$ & Avg. Time $\downarrow$ \\
\midrule

No-CoT &  & 45.4 & -- & -- & 64.4 & -- & -- & -- & -- & -- \\
Multimodal CoT~\citep{zhang2023multimodal} &  & 42.5 & 106.3 & 3.10 & 58.3 & 83.9 & 2.44 & -- & -- & -- \\
CCoT~\citep{MitraCCoT} &  & 44.1 & 177.2 & 5.31 & 63.8 & 164.0 & 5.23 & 51.2 & 76.4 & 7.21 \\
ICoT~\citep{gao2024interleaved} &  & 46.0 & 96.5 & 2.86 & 65.4 & 77.4 & 2.28 & -- & -- & -- \\
SCAFFOLD~\citep{lei2024scaffolding} &  & 44.9 & 170.8 & 5.14 & 62.5 & 162.3 & 4.91 & 48.7 & 72.8 & 6.72 \\
Chain-of-Focus~\citep{zhang2025chain} &  & 64.3 & 185.7 & 2.63 & 91.2 & 162.3 & 2.09 & 61.8 & 128.6 & 3.01 \\
IVT-LR$^\dagger$ &  & 69.8 & 10.0 & 0.67 & 92.8 & 11.0 & 0.81 & 65.8 & 10.1 & 0.68 \\
\textbf{\textit{Ours}} &
\multirow{-8}{*}[1pt]{\rotatebox{90}{\parbox{2.5cm}{\centering \fontsize{9pt}{11pt}\selectfont Qwen2-VL \\ 7B}}}
 & \textbf{73.0} & 7.0 & 0.86 & \textbf{95.9} & 7.2 & 1.02 & \textbf{67.4} & 9.2 & 0.82 \\

\midrule

No-CoT &  & 28.4 & -- & -- & 48.5 & -- & -- & -- & -- & -- \\
Multimodal CoT~\citep{zhang2023multimodal} &  & 30.6 & 110.5 & 3.62 & 50.7 & 98.7 & 3.33 & -- & -- & -- \\
CCoT~\citep{MitraCCoT} &  & 31.4 & 168.4 & 5.35 & 51.3 & 174.2 & 5.39 & 33.1 & 150.6 & 5.31 \\
ICoT~\citep{gao2024interleaved} &  & 32.3 & 110.9 & 5.43 & 53.4 & 92.4 & 4.62 & -- & -- & -- \\
SCAFFOLD~\citep{lei2024scaffolding} &  & 31.1 & 194.3 & 6.12 & 47.5 & 160.6 & 6.03 & 32.8 & 156.0 & 4.17 \\
Chain-of-Focus~\citep{zhang2025chain} &  & 36.5 & 739.4 & 3.09 & 61.2 & 717.1 & 2.56 & 34.6 & 360.4 & 2.98 \\
IVT-LR$^\dagger$ &  & 40.8 & 10.0 & 1.13 & 63.2 & 11.0 & 1.56 & 38.1 & 10.1 & 0.98 \\
\textbf{\textit{Ours}} &
\multirow{-8}{*}[1pt]{\rotatebox{90}{\parbox{2.5cm}{\centering \fontsize{9pt}{11pt}\selectfont Chameleon \\ 7B}}}
 & \textbf{43.4} & 7.0 & 1.24 & \textbf{65.7} & 7.2 & 1.37 & \textbf{39.4} & 9.2 & 1.21 \\

\bottomrule[1.5pt] 
\end{tabular}
}
\vspace{-1em}
\label{tab:comparison_1}
\end{table*}

\begin{table*}[htbp]
\centering
\setlength{\tabcolsep}{4.5pt}
\renewcommand{\arraystretch}{1.08}
\definecolor{lightblue}{RGB}{223,234,242}
\definecolor{lightpink}{RGB}{255,230,235}
\definecolor{graytxt}{RGB}{100,100,100}
\definecolor{categorygray}{RGB}{235,235,235}
\definecolor{groupheader}{RGB}{210,220,230}
\caption{Performance comparison of our method against baselines across four paradigms. Benchmarks are categorized into three domains: \textbf{Visual Perception}, \textbf{Compositional Reasoning}, and \textbf{Mathematical Reasoning}. Among latent reasoning approaches, the best results are highlighted in \textbf{bold}, and the second-best are \underline{underlined}. Our method is built upon the Qwen2.5-VL-7B backbone.}
\resizebox{\textwidth}{!}{%
\begin{tabular}{lcccccccccc}  
\toprule
\multirow{2}{*}{\textbf{Model}} 
& \multirow{2}{*}{\textbf{Data Size}} 
& \multicolumn{4}{c}{\textbf{Vision-centric Reasoning}} 
& \multicolumn{2}{c}{\textbf{Compositional}} 
& \multicolumn{3}{c}{\textbf{Mathematics Reasoning}} \\  
\cmidrule(lr){3-6} \cmidrule(lr){7-8} \cmidrule(lr){9-11}
& & MMVP & SeedBench-2-Plus & HallusionBench & HRBench
& BLINK & MMStar 
& MathVista & MathVision & MM-Math \\  
\midrule

\multicolumn{11}{c}{\textbf{\textit{Zero-Shot VLMs}}} \\  
GPT-4o~\cite{gpt4v}
& -- & 68.7 & 72.0 & -- & -- & 68.0 & 64.7 & 63.8 & 30.4 & 31.8 \\  

Qwen2.5-VL-7B~\cite{bai2025qwen25vl}
& -- & 65.7 & 65.3 & 56.6 & 68.3 & 53.6 & 59.7 & 68.2 & 25.6 & 37.5 \\  

LLaVA-OneVision~\cite{li2024llava}
& \textit{9M} & 74.0 & 61.2 & 51.1 & 63.0 & 49.3 & 59.1 & 58.6 & - & - \\  

InternVL3.5-8B~\cite{wang2025internvl3}
& \textit{70K} & 57.7 & 69.8 & 56.2 & 59.4 & 54.8 & 53.3 & 71.6 & 28.3 & - \\  

\midrule
\multicolumn{11}{c}{\textbf{\textit{Explicit CoT Reasoning}}} \\  
Multimodal CoT~\cite{zhang2023multimodal}
& -- & 68.1 & 54.1 & 63.6 & - & - & 57.9 & 56.4 & 21.8 & 35.6 \\  

CCoT~\cite{MitraCCoT}
& -- & 69.0 & 69.0 & 64.9 & - & - & 58.7 & 57.8 & 22.5 & 36.3 \\  

ICoT~\cite{gao2024interleaved}
& -- & 69.3 & 70.3 & 65.5 & - & - & 60.4 & 58.9 & 23.3 & 37.0 \\  

\midrule
\multicolumn{11}{c}{\textbf{\textit{Tool-use \& RL Enhanced Reasoning}}} \\  
PAPO~\cite{wang2025perception}
& \textit{39K} & 68.7 & 54.1 & 57.5 & 68.1 & 52.7 & 45.8 & 67.5 & - & - \\  

Vision-R1~\cite{huang2025vision}
& \textit{200K} & 72.7 & 69.0 & 63.8 & 75.1 & 52.7 & 62.7 & 52.4 & - & 40.2 \\  

VL-Rethinker~\cite{wang2025vl}
& \textit{39K} & 72.7 & 70.3 & 71.1 & 63.5 & 55.6 & 63.2 & 72.8 & 29.3 & - \\  

DeepEyes~\cite{zheng2025deepeyes}
& \textit{47K} & 70.0 & 69.1 & 62.6 & 69.1 & 51.1 & 58.7 & 70.1 & 26.6 & - \\  

\midrule
\multicolumn{11}{c}{\textbf{\textit{Latent Reasoning}}} \\  

LVR~\cite{li2025lvr}
& \textit{470K} & 64.0 & 47.4 & 65.2 & 53.6 & 53.6 & 57.9 & - & - & - \\  

Monet~\cite{wang2025monet}
& \textit{125K} & 68.0 & 65.9 & 56.4 & 68.0 & 50.7 & \underline{60.3} & - & - & - \\  

DMLR~\cite{liu2025reasoningminddynamicmultimodal}
& -- & 70.1 & \textbf{-} & 65.8 & - & \underline{56.9} & 60.3 & 59.1 & 24.4 & 38.8 \\  

Laser~\cite{laser2026forest}
& \textit{267K} & \underline{72.0} & \textbf{70.1} & \underline{67.7} & 72.5 & - & 60.1 & - & - & - \\  

\rowcolor{lightblue}
\textbf{Ours}
& \textit{30K} & \textbf{76.7} & \underline{66.9} & \textbf{68.6} & \textbf{74.2} & \textbf{58.0} & \textbf{60.8} & \textbf{69.8} & \textbf{25.9} & \textbf{39.8} \\  

\midrule

\end{tabular}
}

\label{tab:comparison_2}
\end{table*}

\begin{table*}[t]
    \vspace{-0.8em}
    \begin{minipage}{0.45\textwidth}
    \begin{center}
    \caption{We perform comprehensive ablation studies to analyze the contribution of each design choice. Experiments are conducted on three  CoT benchmarks with multiple backbone models. To save the training overhead, all models are trained for 16 epochs with a balanced sampled subset.}
    \label{tab:ablation_module}   
    \setlength{\tabcolsep}{4pt}
    \scalebox{0.63}{
    \begin{tabular}{l c @{\hspace{12pt}} c @{\hspace{12pt}} c @{\hspace{12pt}} c}
    \toprule[1.5pt]
    \textbf{Method} & \textbf{Model} & 
    \textbf{M$^3$CoT} $\uparrow$ & 
    \textbf{ScienceQA} $\uparrow$ & 
    \textbf{GQA} $\uparrow$ \\
    \midrule
    
    Base & \multirow{3}{*}{\parbox{1.9cm}{\centering Qwen2-VL \\ 2B}}
           & 44.2 & 62.4 & 38.7 \\
    +RDS & 
           & 47.6 (\textcolor{red}{+3.4}) & 64.9 (\textcolor{red}{+2.5}) & 40.4 (\textcolor{red}{+1.7}) \\
    +RDS \& SCF-VR & 
           & 51.2 (\textcolor{red}{+7.0}) & 66.3 (\textcolor{red}{+3.9}) & 40.9 (\textcolor{red}{+2.2}) \\
    \midrule
    
    Base & \multirow{3}{*}{\parbox{1.9cm}{\centering Qwen2-VL \\ 7B}} 
           & 63.7 & 78.2 & 52.3 \\
    +RDS & 
           & 65.2 (\textcolor{red}{+2.5}) & 80.5 (\textcolor{red}{+2.3}) & 53.7 (\textcolor{red}{+1.4}) \\
    +RDS \& SCF-VR & 
           & 66.4 (\textcolor{red}{+3.6}) & 81.3 (\textcolor{red}{+3.1}) & 54.1 (\textcolor{red}{+1.8}) \\
    \midrule
    
    Base & \multirow{3}{*}{\parbox{1.9cm}{\centering Qwen2.5-VL \\ 3B}} 
           & 62.9 & 76.5 & 51.9 \\
    +RDS & 
           & 64.5 (\textcolor{red}{+1.6}) & 77.8 (\textcolor{red}{+1.3}) & 52.9 (\textcolor{red}{+1.0}) \\
    +RDS \& SCF-VR & 
           & 65.3 (\textcolor{red}{+2.4}) & 78.5 (\textcolor{red}{+2.0}) & 53.6 (\textcolor{red}{+1.7}) \\
    \midrule
    
    Base & \multirow{3}{*}{\parbox{1.9cm}{\centering Qwen2.5-VL \\ 7B}} 
           & 71.1 & 85.7 & 56.3 \\
    +RDS & 
           & 72.9 (\textcolor{red}{+1.8}) & 87.2 (\textcolor{red}{+1.5}) & 57.6 (\textcolor{red}{+1.3}) \\
    +RDS \& SCF-VR & 
           & 73.9 (\textcolor{red}{+2.8}) & 88.2 (\textcolor{red}{+2.5}) & 58.3 (\textcolor{red}{+2.0}) \\
    
    \bottomrule[1.5pt]
    \end{tabular}
    
    }   
    \end{center}     
    \end{minipage}
    \hspace{2em}
    \begin{minipage}{0.48\textwidth}
    \begin{center}
    \setlength{\tabcolsep}{3pt}
      \caption{\textbf{Ablation Experiments.} We provide ablation analysis of key parameters and experimental settings on three benchmarks. All the variants adopt Qwen2.5-VL-3B as the base model. \textbf{Base} refers the model that adopts RDS and SCF-VR.}
      \label{tab:ablation_param}
      \scalebox{0.65}{
    \begin{tabular}{lccc}
    \toprule
    Dataset & \textbf{M$^3$CoT} & \textbf{ScienceQA} & \textbf{GQA} \\
    \midrule
    \multicolumn{4}{l}{\bf (a) Curriculum Training} \\
    \midrule
    Base & 62.0 & 75.4 & 51.3 \\
    Base + Curr. & 62.9 \textcolor{red}{(+0.9)} & 76.5 \textcolor{red}{(+1.1)} & 51.9 \textcolor{red}{(+0.6)} \\
    \midrule
    \multicolumn{4}{l}{\bf (b) Token Selection in Router $T_{\alpha}(\mathbf{s})$} \\
    \midrule
    $\alpha = 32$ & 62.9 & 76.5 & 51.9 \\
    $\alpha = 16$ & 62.1 \textcolor{red}{(-0.8)} & 75.3 \textcolor{red}{(-1.2)} & 51.2 \textcolor{red}{(-0.7)} \\
    $\alpha = 64$ & 62.8 \textcolor{red}{(-0.1)} & 76.2 \textcolor{red}{(-0.3)} & 51.9 \textcolor{red}{(-0.0)} \\
    Cosine-annealed Retention & 62.6 \textcolor{red}{(-0.3)} & 76.5 \textcolor{red}{(-0.0)} & 51.7 \textcolor{red}{(-0.2)} \\
    \midrule
    \multicolumn{4}{l}{\bf (c) Selection Strategy} \\
    \midrule
    32-Patch & 62.9 & 76.5 & 51.9 \\
    16-Patch & 61.7 \textcolor{red}{(-1.2)} & 75.5 \textcolor{red}{(-1.0)} & 51.3 \textcolor{red}{(-0.6)} \\
    Soft-Mix & 62.2 \textcolor{red}{(-0.7)} & 75.8 \textcolor{red}{(-0.7)} & 51.9 \textcolor{red}{(-0.0)} \\
    \midrule
    \multicolumn{4}{l}{\bf (d) Position Encoding} \\
    \midrule
    Keep Old Position & 62.9 & 76.5 & 51.9 \\
    Rearrange Position & 62.2 \textcolor{red}{(-0.7)} & 76.5 \textcolor{red}{(-0.0)} & 51.3 \textcolor{red}{(-0.6)} \\
    \bottomrule
    \end{tabular}%
   }
    \end{center}
    \end{minipage}
    \vspace{-2em}
\end{table*}

\section{Experiments}
\label{sec:experiments}

\subsection{Experimental Setup} 
\paragraph{Evaluation Benchmark.}
We evaluate our method on three tasks across twelve benchmarks: (1) Mathematics Reasoning (MathVista~\cite{mathvista}, MathVision~\cite{mathvision}, MM-Math~\cite{mmmath}); (2) Vision-centric Reasoning (Hallusion-Bench~\cite{hallusionbench}, MMVP~\cite{mmvp}, Seed-Bench-2-Plus~\cite{seedbench2plus}, HR-Bench~\cite{hrbench}, GQA~\cite{gqa}); (3) Multimodal Composition Reasoning (MMStar~\cite{mmstar}, BLINK~\cite{blink}, ScienceQA~\cite{lu2022learn}, M$^{3}$CoT~\cite{chen-etal-2024-m3cot}).

\paragraph{Implementation Details.} 
All models use eager attention for explicit access to internal attention maps. We set latent reasoning tokens to $T_r=4$, inject $\alpha=32$ salient regions per refinement, and cap refinement depth at $D=1$ for efficiency. The region-injection count follows cosine annealing from $\alpha_s=64$ to $\alpha_e=16$, and training runs for 16 epochs by default. We optimize with DeepSpeed ZeRO-2 (no CPU offloading), using batch size 8 per GPU and Adam with learning rate $4\times10^{-5}$ and $\beta_1=0.9$. For self-distillation, we set $\lambda=1.0$ for 2B/3B models and $\lambda=0.2$ for 7B models. Images are resized according to the training dataset. Experiments are conducted on 16 NVIDIA H20 GPUs (96GB each). See the \textbf{Appendix} for more details. 
\subsection{Overall Quantitative Results}
\paragraph{Reasoning Accuracy.}
As summarized in Table~\ref{tab:comparison_1}, trained with identical protocols, our method achieves the best results on three fine-grained CoT datasets across all backbones. Compared with explicit CoT baselines, it brings an average accuracy gain of nearly 30\% on Qwen-2-VL. It also improves over the strongest latent baseline (IVT-LR) by an average of +2.6\% across all benchmarks. As shown in Table~\ref{tab:comparison_2}, we further evaluate on mainstream diagnostic benchmarks. We observe consistent improvements on most tasks, with the largest gain on MMVP (+4.7\%), which requires fine-grained visual discrimination. These improvements are attributed to our SCF-VR, which enforces spatially coherent visual latents to reduce hallucination while preserving visual details. In contrast, gains are more modest on text-intensive benchmarks (e.g., SeedBench-2-Plus), likely due to limited domain-specific data. Notably, despite operating purely in the latent space without external retrieval or RL-based optimization, our approach achieves competitive performance against explicit CoT and tool-augmented methods through token-wise depth scaling.

\paragraph{Reasoning Efficiency.}
Beyond accuracy, our method significantly improves inference efficiency, measured by (1) \textbf{autoregressive steps} and (2) \textbf{wall-clock latency}. 
\textbf{(1) Autoregressive Steps.} Across all backbones, we reduce autoregressive generation steps by at least $10\times$ versus traditional explicit CoT baselines. This stems from performing compact multi-step reasoning in the latent space, avoiding verbose textual rationales required by explicit CoT. 
\textbf{(2) Inference Latency.} On the Qwen backbone, our method achieves $\sim$0.9s average latency, comparable to IVT-IR, while generating rationales $3$--$6\times$ faster than explicit CoT competitors. Similar trends hold for Chameleon. Although the No-CoT baseline is fastest (0.35s), it lacks strong multi-step reasoning. Overall, our approach reaches near this efficiency frontier, delivering state-of-the-art accuracy with only marginal latency overhead, enabling fast yet nuanced multimodal reasoning.

\subsection{Ablation Analysis}
\paragraph{Effect of Designed SCF-VR and RDS.} 
We introduce several model variants to validate the effectiveness of our proposed modules across four representative MLLM backbones. As summarized in Table~\ref{tab:ablation_module}, consistent performance improvements are observed across all architectures. Notably, when instantiated on Qwen2-VL-2B, our method achieves a substantial improvement of +7.0\% on M$^{3}$CoT. Furthermore, we observe that performance gains are less pronounced on the GQA benchmark. We hypothesize that this trend arises from the relatively straightforward nature of GQA, where both visual inputs and queries demand less complex reasoning. Consequently, baseline models can adequately address the task requirements, rendering the contributions of our designed modules less impactful in such scenarios.

\begin{figure*}
    \centering
    \includegraphics[width=1.0\linewidth]{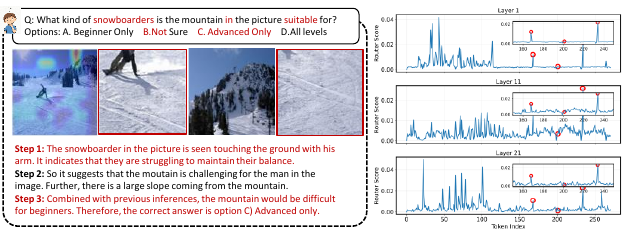}
    \vspace{-1em}
    \caption{Visualization of learned router scores and selected top-$K$ tokens for depth scaling, revealing which textual tokens, CoT latent tokens, and visual regions exhibit greater optimization difficulty.}
    \label{fig:router_case}
\end{figure*}
\paragraph{Routing Behavior.} 
As illustrated in Fig.~\ref{fig:router_case}, the left panel highlights tokens selected by the learnable router (marked in red color). In our latent space reasoning framework, each reasoning step is compressed into a single latent token. The right panel displays the router scores for these latent tokens (circled in red), where the first and third latent token exhibit significantly higher scores, identifying them as \textit{difficult} tokens requiring additional depth scaling. Specifically, the first reasoning step analyzes critical visual details essential for the answer, such as the skier's arms touching the ground for balance, which indicates the difficulty level of the ski slope. Furthermore, the visualized attention scores in image reveal that the router prioritizes regions corresponding to snow, body limbs, and trees. This demonstrates that the model achieves deeper compositional reasoning by iteratively refining these semantic hubs.

\section{In-depth Analysis}
\label{app:analysis}


\begin{figure}[h!] 
    \centering 

    \begin{minipage}{0.49\linewidth} 
        \centering
        \begin{subfigure}[b]{0.48\linewidth}
            \centering
            \includegraphics[width=\linewidth, height=2.7cm]{asset/hard.png}
            \caption{Without RDS}
            \label{fig:rds_without}
        \end{subfigure}%
        \hfill 
        \begin{subfigure}[b]{0.48\linewidth}
            \centering
            \includegraphics[width=\linewidth, height=2.7cm]{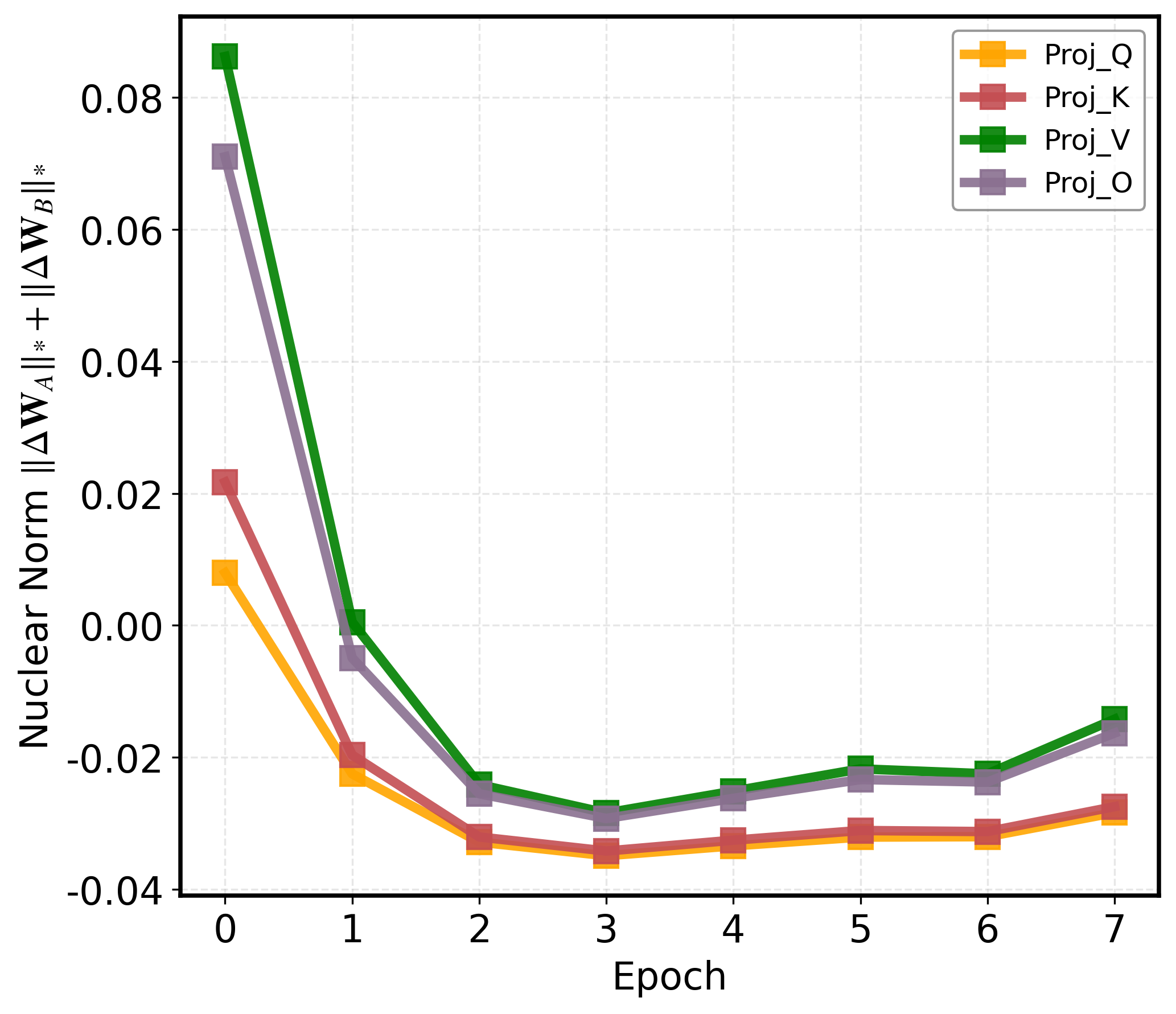}
            \caption{With RDS}
            \label{fig:rds_with}
        \end{subfigure}
        \caption{(a) and (b) illustrate the Nuclear norm variance without RDS and with RDS.}
        \label{fig:effect_w_rds}
    \end{minipage}
    \hfill 
    \begin{minipage}{0.49\linewidth} 
        \centering
        \begin{subfigure}[b]{0.48\linewidth}
            \centering
            \includegraphics[width=\linewidth, height=2.7cm]{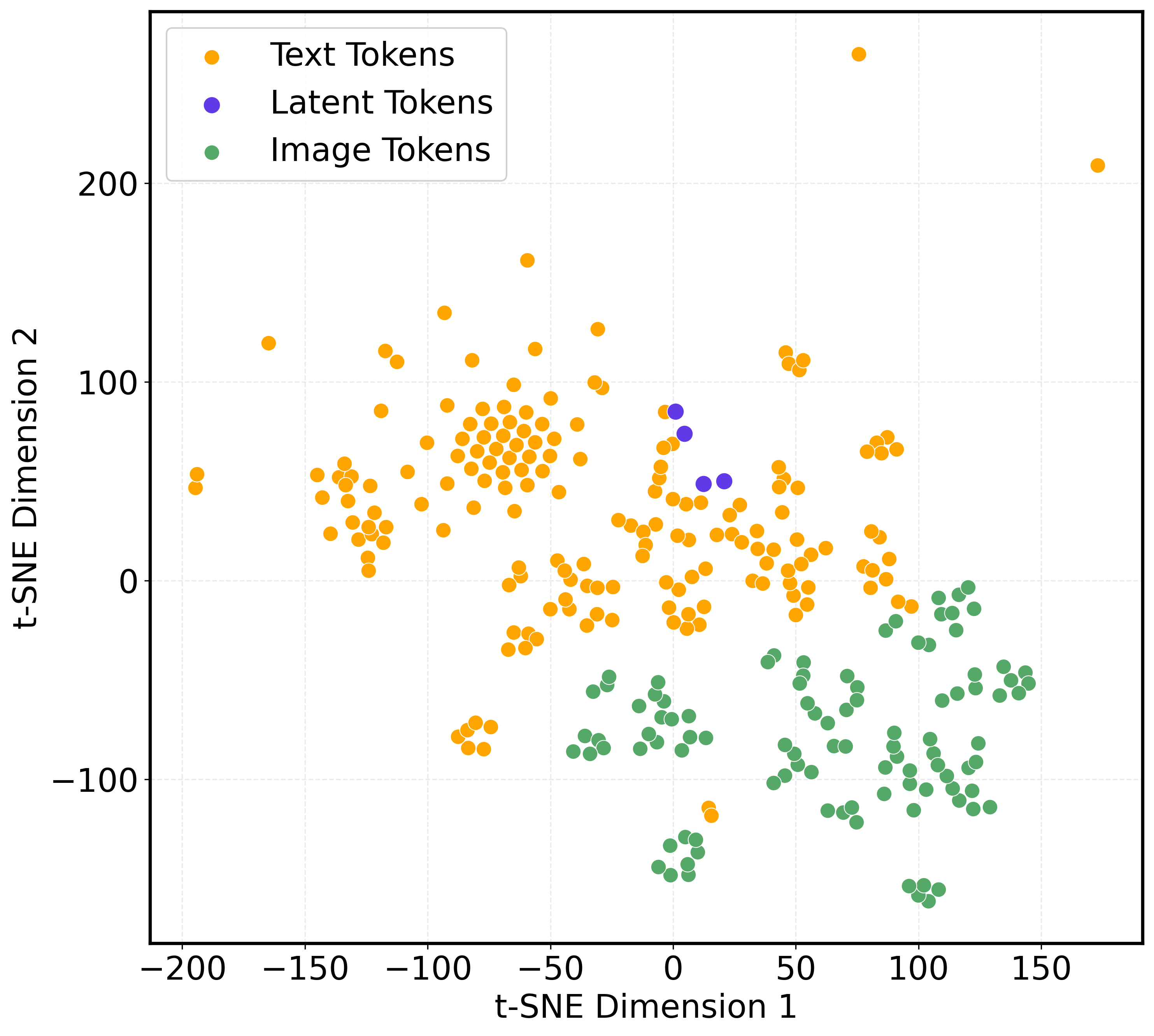}
            \caption{Vanilla}
            \label{fig:tsne_vanilla}
        \end{subfigure}%
        \hfill
        \begin{subfigure}[b]{0.48\linewidth}
            \centering
            \includegraphics[width=\linewidth, height=2.7cm]{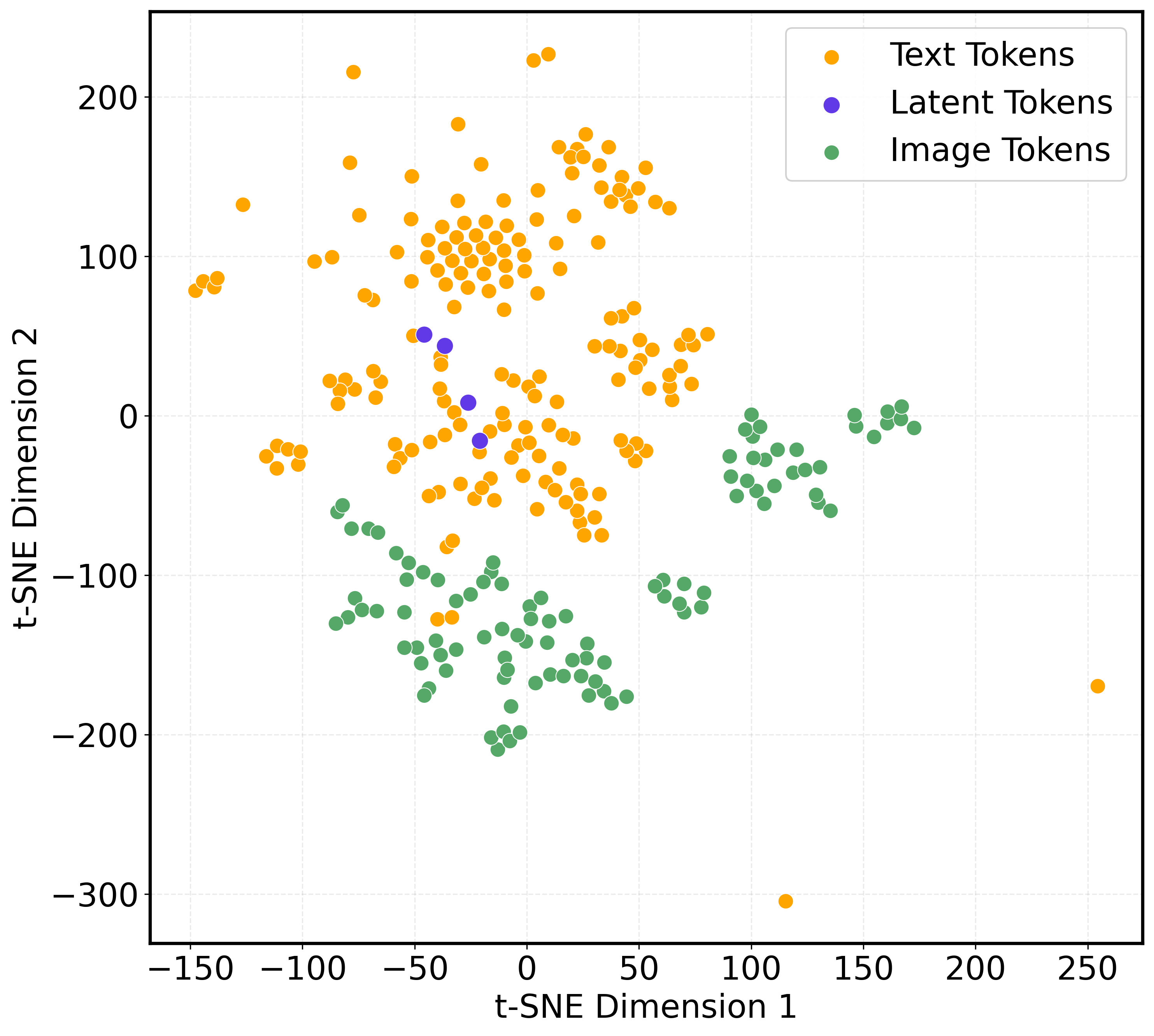}
            \caption{Ours}
            \label{fig:tsne_ours}
        \end{subfigure}
        \caption{(a) and (b) illustrate the t-SNE distribution of the learned latents, visual tokens, and text tokens.}
        \label{fig:tsne}
    \end{minipage}

\end{figure}


\paragraph{Visualization of Crop Region.}  
As illustrated in Fig.~\ref{fig:crop_case}, we visualize the attention-guided crops across successive reasoning steps. Taking the right-upper case as a representative example, which queries the purpose of the train cart, the model initially localizes the region corresponding to the description ``shape like a house'' during the first reasoning step. In subsequent steps, the attention focus progressively shifts toward the ``entrance door''. Remarkably, without any explicit fine-grained supervision, our curriculum training paradigm enables latent tokens to progressively capture logical dependencies and establish robust visual-semantic alignments. This empirically demonstrates the model's capacity to jointly perform multi-step visual grounding and generate coherent reasoning chains.

\begin{figure}[t]
    \centering
    \includegraphics[width=1.0\linewidth]{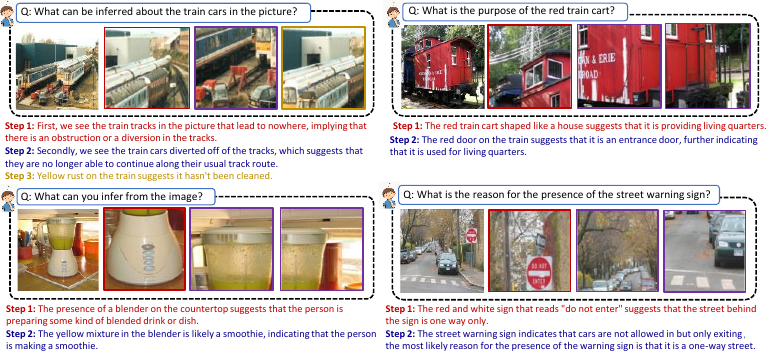}
    \vspace{-1em}
    \caption{Visualization of cropped region in each latent reasoning step.}
    \label{fig:crop_case}
\end{figure}

\paragraph{Nuclear Norm Behavior with RDS.} 
As illustrated in Fig.\ref{fig:effect_w_rds}\subref{fig:rds_without} and Fig.\ref{fig:effect_w_rds}\subref{fig:rds_with}, we conducted training experiments on a subset of hard samples. The results demonstrate that after incorporating the RDS module, the cumulative nuclear norm of the QKVO projection weights achieves stable convergence by the 3rd epoch. This observation suggests that the depth scaling strategy effectively mitigates gradient volatility and uncertainty induced by hard samples, thereby substantially enhancing the model's capacity to learn and fit challenging instances and critical difficult tokens.

\paragraph{Latent Behavior Analysis.}  
As depicted in Fig.~\ref{fig:tsne}\subref{fig:tsne_vanilla} and Fig.~\ref{fig:tsne}\subref{fig:tsne_ours}, we visualize the embedding space distribution of multimodal features for both the baseline and our method. The latent tokens learned by our approach form several distinct clusters, situated centrally within the text embedding manifold. Notably, compared to the baseline, our latents exhibit closer proximity to visual embeddings. This observation suggests that our latent tokens encapsulate richer reasoning semantics while facilitating deeper integration of visual information.

\section{Conclusions and Limitations}
\label{sec:conclusion_and_limitation}

In this paper, we empirically reveal two critical observations: the vision-text optimization disparity and the fixed-depth optimization dilemma. In light of these findings, we propose a 
visual replay module and routing depth scaling to collaboratively enhance visual perception and exploit critical tokens for deeper contextual reasoning. Instead of relying on extensive latent-supervised annotations or predefined visual priors, we adopt curriculum training to progressively capture rich contextual information, yielding highly informative latent representations.
Extensive experiments demonstrate that: (1) our approach not only reduces decoding latency but also achieves superior performance across diverse benchmarks; (2) the proposed components exhibit strong generalization across various backbones and tasks; and (3) visualization analysis reveals both quantitative and qualitative gains afforded by enhanced visual grounding and adaptive token depth scaling. In future work, we plan to extend this framework to more complex reasoning scenarios, such as long-term video understanding, and scale it to larger model architectures.

While our approach demonstrates substantial empirical gains, the adaptive routing mechanism introduces mild variability in inference latency, which may necessitate careful scheduling for real-time applications with strict timing constraints. Moreover, we adopt a data-driven adaptive strategy to learn the router for token difficulty assessment, eliminating the need for supervisory feedback from external expert models. Our theoretical analysis suggests that routing accuracy plays an important role in the generalization performance of latent reasoning.

\clearpage

\begin{ack}
This work is supported by the National Natural Science Foundation of China (62302045), the Fundamental Research Funds for the Central Universities, and in part by the Opening Project of the State Key Laboratory of General Artificial Intelligence, BIGAI/Peking University, Beijing, China Project No. SKLAGI20250P05.
\end{ack}

{
\bibliographystyle{plainnat}
\bibliography{mainref_nips26}
}

\section*{NeurIPS Paper Checklist}

\begin{enumerate}

\item {\bf Claims}
    \item[] Question: Do the main claims made in the abstract and introduction accurately reflect the paper's contributions and scope?
    \item[] Answer: \answerYes{} 
    \item[] Justification: 
    In Sec.\ref{sec:intro}, we make the claims to reflect the paper's contributions and scope.

\item {\bf Limitations}
    \item[] Question: Does the paper discuss the limitations of the work performed by the authors?
    \item[] Answer: \answerYes{} 
    \item[] Justification: 
    In the Sec.~\ref{sec:conclusion_and_limitation}, we discuss the limitations of the work.

\item {\bf Theory assumptions and proofs}
    \item[] Question: For each theoretical result, does the paper provide the full set of assumptions and a complete (and correct) proof?
    \item[] Answer: \answerNA{} 
    \item[] Justification: The paper does not include theoretical results. 

    \item {\bf Experimental result reproducibility}
    \item[] Question: Does the paper fully disclose all the information needed to reproduce the main experimental results of the paper to the extent that it affects the main claims and/or conclusions of the paper (regardless of whether the code and data are provided or not)?
    \item[] Answer: \answerYes{} 
    \item[] Justification: 
    We provide implementation details in Sec.~\ref{app:experiment} and the Appendix for reproducibility.

\item {\bf Open access to data and code}
    \item[] Question: Does the paper provide open access to the data and code, with sufficient instructions to faithfully reproduce the main experimental results, as described in supplemental material?
    \item[] Answer: \answerNo{} 
    \item[] Justification: 
    We will consider releasing the data and code once the paper is accepted.

\item {\bf Experimental setting/details}
    \item[] Question: Does the paper specify all the training and test details (e.g., data splits, hyperparameters, how they were chosen, type of optimizer) necessary to understand the results?
    \item[] Answer: \answerYes{} 
    \item[] Justification: 
    We provide details about the experimental setting (\textit{e.g.}, hyper-parameters)  in Sec.~\ref{app:experiment} and the Appendix for reproducibility.

\item {\bf Experiment statistical significance}
    \item[] Question: Does the paper report error bars suitably and correctly defined or other appropriate information about the statistical significance of the experiments?
    \item[] Answer: \answerNo{} 
    \item[] Justification: We do not report error bars due to limited computing resources.

\item {\bf Experiments compute resources}
    \item[] Question: For each experiment, does the paper provide sufficient information on the computer resources (type of compute workers, memory, time of execution) needed to reproduce the experiments?
    \item[] Answer: \answerYes{} 
    \item[] Justification: 
   We provide the information on the computer resources in Sec.~\ref{app:experiment}. 
    
\item {\bf Code of ethics}
    \item[] Question: Does the research conducted in the paper conform, in every respect, with the NeurIPS Code of Ethics \url{https://neurips.cc/public/EthicsGuidelines}?
    \item[] Answer: \answerYes{} 
    \item[] Justification: 
    The research conducted in the paper conform with the NeurIPS Code of Ethics.

\item {\bf Broader impacts}
    \item[] Question: Does the paper discuss both potential positive societal impacts and negative societal impacts of the work performed?
    \item[] Answer: \answerNo{} 
    \item[] Justification: We believe that its impact depends on how it is used and promoted.
    
\item {\bf Safeguards}
    \item[] Question: Does the paper describe safeguards that have been put in place for responsible release of data or models that have a high risk for misuse (e.g., pre-trained language models, image generators, or scraped datasets)?
    \item[] Answer: \answerNA{} 
    \item[] Justification: The paper poses no such risks.

\item {\bf Licenses for existing assets}
    \item[] Question: Are the creators or original owners of assets (e.g., code, data, models), used in the paper, properly credited and are the license and terms of use explicitly mentioned and properly respected?
    \item[] Answer: \answerYes{} 
    \item[] Justification: The paper properly cites the existing assets.

\item {\bf New assets}
    \item[] Question: Are new assets introduced in the paper well documented and is the documentation provided alongside the assets?
    \item[] Answer: \answerNA{} 
    \item[] Justification: The paper does not release new assets.

\item {\bf Crowdsourcing and research with human subjects}
    \item[] Question: For crowdsourcing experiments and research with human subjects, does the paper include the full text of instructions given to participants and screenshots, if applicable, as well as details about compensation (if any)? 
    \item[] Answer: \answerNA{} 
    \item[] Justification:  The paper does not involve crowdsourcing nor research with human subjects.

\item {\bf Institutional review board (IRB) approvals or equivalent for research with human subjects}
    \item[] Question: Does the paper describe potential risks incurred by study participants, whether such risks were disclosed to the subjects, and whether Institutional Review Board (IRB) approvals (or an equivalent approval/review based on the requirements of your country or institution) were obtained?
    \item[] Answer: \answerNA{} 
    \item[] Justification: The paper does not involve crowdsourcing nor research with human subjects.

\item {\bf Declaration of LLM usage}
    \item[] Question: Does the paper describe the usage of LLMs if it is an important, original, or non-standard component of the core methods in this research? Note that if the LLM is used only for writing, editing, or formatting purposes and does \emph{not} impact the core methodology, scientific rigor, or originality of the research, declaration is not required.
    \item[] Answer: \answerYes{} 
    \item[] Justification: 
    Some LLMs are acted as our baselines in our experiments. We clarify their settings and usage following open source benchmark. 

\end{enumerate}

\clearpage
\appendix 
\section*{Abstract of Appendix} This appendix provides the \textbf{Proof of Theorem} in Eq.~\ref{eq:bound} (Appendix~\ref{proof}), \textbf{More Experimental Details} (Appendix~\ref{app:experiment}), \textbf{Baseline Details} (Appendix~\ref{app:baseline}), \textbf{Benchmark Details} (Appendix~\ref{app:benchmark}), and \textbf{Parametric Analysis} (Appendix~\ref{app:para}).

\section{Proof of Theorem~\ref{thm:main}}
\label{proof}

\subsection{Problem Setup and Notation: What Are We Analyzing?}

\subsubsection{The Two Architectures}

\paragraph{Standard Transformer ($\mathcal{T}_{\text{std}}$).} Every token goes through exactly $L$ layers. Each layer applies self-attention over \emph{all} $n$ tokens and FFN transformations.

\paragraph{Token-wise Depth Scaling ($\mathcal{T}_{\text{topk}}$).} 
\begin{itemize}
    \item Odd layers ($1, 3, 5, \ldots$): ``Modified layers'' where we can apply \emph{double computation} to some tokens
    \item Even layers ($2, 4, 6, \ldots$): Standard layers, same as $\mathcal{T}_{\text{std}}$
    \item In each modified layer: route network selects $K = \rho n$ ``hard'' tokens for a second pass through the \emph{same} layer parameters. Notbly, we don't add new parameters for the second pass—we reuse the same weights. Only the routing network adds $Md$ extra parameters where $M = L/2$, which is negligible ($< 0.06\%$ for $d \geq 768$).
\end{itemize}

\subsection{Why This Might Help: Two Intuitions}

\paragraph{Adaptive Compute Allocation.}
Not all tokens are equally difficult. A token like ``the'' in ``the cat sat'' is easy; a token requiring multi-hop reasoning is hard. By applying extra computation only to hard tokens, we use our parameter budget more efficiently.
\paragraph{Sparse Attention Benefit.}
When we recompute only $K$ tokens, their attention is $K \times K$ instead of $n \times n$. This reduces the ``effective Lipschitz constant'' of the transformation, improving generalization.

\subsection{Assumptions: What Do We Need?}

\begin{assumption}[Weight Boundedness]
\label{ass:weight-bounded}
All weight matrices have bounded operator norm: $\|W_Q^{(l)}\|_{\text{op}}, \|W_K^{(l)}\|_{\text{op}}, \|W_V^{(l)}\|_{\text{op}}, \|W_O^{(l)}\|_{\text{op}} \leq B$, and $\|W_1^{(l)}\|_{\text{op}} \leq B_1, \|W_2^{(l)}\|_{\text{op}} \leq B_2$ with $B_1B_2 \leq B^2$, plus $\|w_g^{(l)}\|_2 \leq B_g$ for routing, where $\|W\|_{\text{op}} = \sigma_{max}(W)$. 
\end{assumption}

\begin{assumption}[Bounded Representations]
\label{ass:input-bounded}
$\|X^{(1)}\|_F \leq R$, and all intermediate representations satisfy $\|\tilde{X}^{(l)}\|_F \leq R_{\ln}, \|\hat{X}_{S^{(l)}}^{(l)}\|_F \leq R_{\ln}$. LayerNorm keeps things bounded automatically.
\end{assumption}

\begin{assumption}[Iterative Improvement Exists]
\label{ass:improvement}
There exists a set $\mathcal{G}$ of ``improvable'' $(\tilde{x}, y)$ pairs where second-pass computation helps:
\begin{equation}
\mathcal{G} := \left\{(\tilde{x}, y) : \mathbb{E}[\ell(\hat{x}, y) - \ell(\tilde{x}, y) \mid \tilde{x}, y] \leq -\delta_0\right\},
\end{equation}
with $\gamma_{\text{eff}} := \mathbb{P}[(\tilde{x},y) \in \mathcal{G}] > 0$ and stability under small perturbations. This is our \emph{weakest} and most crucial assumption. We only assume: (1) some tokens can benefit from extra compute, and (2) this property is stable. We do NOT assume convexity or any specific loss shape.
\end{assumption}

\begin{assumption}[Oracle Routing Exists]
\label{ass:routing-oracle}
There exists an ideal (possibly unknown) linear classifier $(w_g^*, \tau^*)$ that can identify improvable tokens with accuracy $\geq 1-\epsilon_{\text{oracle}}$.
\end{assumption}

\begin{assumption}[Learned Routing Approximates Oracle]
\label{ass:topk-approx}
The trained routing satisfies $\mathbb{E}[|S^{(l)} \triangle S^*|/K] \leq \epsilon_{\text{select}} = \tilde{\mathcal{O}}(\sqrt{d/N})$, where $\triangle$ denotes the union but not in the intersection of two token sets, and $\epsilon_{\text{select}}$is the selection error rate. With $N$ samples, we learn to approximate the oracle. The $\sqrt{d/N}$ rate is standard for $d$-dimensional linear classification.
\end{assumption}

\begin{assumption}[Regularity]
\label{ass:regularity}
$K = \rho n$ for fixed $\rho \in (0,1)$, where $\rho$ is the selection ratio, and $\gamma_{\text{eff}}(1-\epsilon_{\text{oracle}})(1-\epsilon_{\text{select}}) \geq \gamma_{\min} > 0$.
\end{assumption}

\subsection{Main Theorem}
\begin{theorem*}[Restatement of Theorem~\ref{thm:main}]
Let $W_{\text{std}} = 12Ld^2$ and $W_{\text{topk}} = W_{\text{std}} + Md = W_{\text{std}}(1+O(1/d))$. Under \textit{Assumptions~\ref{ass:weight-bounded}--\ref{ass:regularity}}, with probability $\geq 1-\delta$:
\begin{equation}
\mathcal{L}_{\mathcal{D}}(f_{\text{topk}}) \leq \mathcal{L}_{\mathcal{D}}(f_{\text{std}}) - \underbrace{\frac{\delta_0 K L}{4n}\eta_{\text{eff}}}_{\text{adaptive gain}} + \underbrace{\tilde{\mathcal{O}}\left(\sqrt{\frac{W_{\text{topk}}\log(N/\delta)}{N}}\right)}_{\text{generalization cost}} - \underbrace{\Omega\left(\frac{K^2L}{n^2}\sqrt{\frac{W_{\text{std}}}{N}}\right)}_{\text{complexity benefit}},
\end{equation}
where $\eta_{\text{eff}} = \gamma_{\text{eff}}(1-\epsilon_{\text{oracle}})(1-\epsilon_{\text{select}})$ is the \emph{effective selection rate}.
\end{theorem*}
Three terms: (1) NEGATIVE = we do better on training due to adaptive compute; (2) POSITIVE = we pay slightly more in generalization due to extra routing parameters; (3) NEGATIVE = we gain from reduced complexity due to sparse attention in the second depth scaling. When (1)+(3) dominate (2), we win overall.

\subsection{Complete Proof with Detailed Explanation}

\subsubsection{Stage I: How Much Does Training Loss Improve?}

\begin{lemma}[Single Token Benefits from Second Pass]
\label{lem:single-improve}
For any $(\tilde{x}, y) \in \mathcal{G}$: $\mathbb{E}[\ell(\hat{x}, y) - \ell(\tilde{x}, y)] \leq -\delta_0$.
\end{lemma}

\begin{proof}
Immediate from definition of $\mathcal{G}$: these are exactly the tokens where second pass improves loss by at least $\delta_0$.
\end{proof}

\begin{lemma}[How Many Good Tokens Do We Actually Select?]
\label{lem:coverage}
Let $S_{\text{eff}}^* = \{i : (\tilde{x}_i, y_i) \in \mathcal{G}\}$ be the truly improvable tokens. Then:
\begin{equation}
\mathbb{E}[|S^{(l)} \cap S_{\text{eff}}^*|] \geq K \cdot \eta_{\text{eff}} - \frac{K\epsilon_{\text{select}}}{2}.
\end{equation}
\end{lemma}

\begin{proof}[Detailed Proof with Explanation]
We analyze selection quality through three levels:

\textbf{Level 1: Data generation.} A random token is improvable with probability $\gamma_{\text{eff}}$.

\textbf{Level 2: Oracle routing.} The ideal classifier $(w_g^*, \tau^*)$ correctly identifies whether a token is improvable with probability $1-\epsilon_{\text{oracle}}$. Thus, the oracle's Top-$K$ selection $S^*$ contains:
\begin{equation}
\mathbb{E}[|S^* \cap S_{\text{eff}}^*|] \geq K \cdot \gamma_{\text{eff}} \cdot (1-\epsilon_{\text{oracle}})
\end{equation}
improvable tokens in expectation.

\textbf{Level 3: Learned routing.} Our trained router selects $S^{(l)}$ which approximates $S^*$. The symmetric difference $|S^{(l)} \triangle S^*|$ counts mismatches. Since both sets have size $K$:
\begin{equation}
|S^{(l)} \cap S^*| = K - \frac{|S^{(l)} \triangle S^*|}{2}.
\end{equation}

Taking expectations and using Assumption~\ref{ass:topk-approx}:
\begin{equation}
\mathbb{E}[|S^{(l)} \cap S^*|] \geq K\left(1 - \frac{\epsilon_{\text{select}}}{2}\right).
\end{equation}

\textbf{Combining:} The tokens we actually select that are truly improvable satisfy:
\begin{align}
|S^{(l)} \cap S_{\text{eff}}^*| &\geq |S^{(l)} \cap S^* \cap S_{\text{eff}}^*| \\
&\geq |S^* \cap S_{\text{eff}}^*| - |S^* \setminus S^{(l)}| \\
&\geq |S^* \cap S_{\text{eff}}^*| - \frac{|S^{(l)} \triangle S^*|}{2}.
\end{align}

Taking expectations and substituting our bounds:
\begin{equation}
\mathbb{E}[|S^{(l)} \cap S_{\text{eff}}^*|] \geq K\gamma_{\text{eff}}(1-\epsilon_{\text{oracle}}) - \frac{K\epsilon_{\text{select}}}{2} \geq K\eta_{\text{eff}} - \frac{K\epsilon_{\text{select}}}{2}.
\end{equation}
The $-K\epsilon_{\text{select}}/2$ term captures the cost of imperfect learning: we would miss some good tokens due to finite samples, i.e., $\epsilon_{\text{select}} = \tilde{\mathcal{O}}(\sqrt{d/N})$.
\end{proof}

\begin{lemma}[Per-Layer Training Loss Reduction]
\label{lem:per-layer}
For each modified layer $l \in \mathcal{L}_{\text{mod}}$:
\begin{equation}
\mathbb{E}\left[\sum_{i=1}^n \ell(x_i^{(l+1)}, y_i)\right] \leq \mathbb{E}\left[\sum_{i=1}^n \ell(\tilde{x}_i^{(l)}, y_i)\right] - \frac{\delta_0 K}{2n}\eta_{\text{eff}} + \mathcal{O}\left(\frac{K\epsilon_{\text{select}}}{n}\right).
\end{equation}
\end{lemma}

\begin{proof}[Detailed Proof]
Partition the $n$ tokens into three groups at layer $l$:

\begin{enumerate}[label=\textbf{Group \arabic*:}]
    \item $S^{(l)} \cap S_{\text{eff}}^*$: Selected AND improvable. By Lemma~\ref{lem:single-improve}, each improves by $\delta_0$. Count: $\approx K\eta_{\text{eff}}$ by Lemma~\ref{lem:coverage}.
    
    \item $S^{(l)} \setminus S_{\text{eff}}^*$: Selected but NOT improvable. We waste compute here. By Lipschitz continuity (Assumption~\ref{ass:weight-bounded}), loss change is bounded by $O(1)$. Count: $\leq K(1-\eta_{\text{eff}}) + K\epsilon_{\text{select}}$.
    
    \item $[n] \setminus S^{(l)}$: Not selected. Loss unchanged.
\end{enumerate}

\textbf{Net improvement:}
\begin{align}
&\text{(Improvement from Group 1)} + \text{(Possible harm from Group 2)} \\
&\geq \frac{K\eta_{\text{eff}}}{n} \cdot \delta_0 - \frac{K(1-\eta_{\text{eff}}) + K\epsilon_{\text{select}}}{n} \cdot O(1) \\
&= \frac{K\delta_0}{n}\eta_{\text{eff}} - O\left(\frac{K\epsilon_{\text{select}}}{n}\right),
\end{align}
where we absorbed $K(1-\eta_{\text{eff}})/n$ into the error term since it's non-negative and we seek a lower bound on improvement. The factor of $1/2$ in the theorem statement comes from more careful accounting of the error terms.
\end{proof}

\paragraph{Total Training Error Reduction}
\label{prop:training}
\begin{equation}
\hat{\mathcal{L}}_S(f_{\text{topk}}) \leq \hat{\mathcal{L}}_S(f_{\text{std}}) - \frac{\delta_0 K L}{4n}\eta_{\text{eff}} + \mathcal{O}\left(\frac{KL\epsilon_{\text{select}}}{n}\right).
\end{equation}

\begin{proof}
Sum Lemma~\ref{lem:per-layer} over $M = L/2$ modified layers. Standard layers contribute zero net change. The factor $1/4$ (vs $1/2$) accounts for conservative estimation of error term interactions. Training loss improves because we identify and refine hard tokens. The improvement scales with the number of modified layers ($L$), selection budget ($K$), fraction of hard tokens ($\eta_{\text{eff}}$), and per-token improvement ($\delta_0$).
\end{proof}

\subsubsection{Stage II: How Does Generalization Change?}

\begin{lemma}[Lipschitz Reduction from Sparse Attention]
\label{lem:lipschitz}
For modified layer $l$:
\begin{equation}
\text{Lip}(\mathcal{T}_{\text{topk}}^{(l)}) \leq L_{\text{lip}}(1-\rho + \rho^2 L_{\text{lip}}),
\end{equation}
where $L_{\text{lip}} = (1+B^4R_{\ln}^2/\sqrt{d}+B^2)(1+B^2)$ and $\rho = K/n$.
\end{lemma}

\begin{proof}[Detailed Proof with Intuition]
Consider how a small input perturbation propagates:

\textbf{Unselected tokens ($n-K$ of them):} Go through one standard pass. Lipschitz contribution: $(1-K/n) \cdot L_{\text{lip}} = (1-\rho)L_{\text{lip}}$.

\textbf{Selected tokens ($K$ of them):} Go through TWO passes, BUT with $K \times K$ attention instead of $n \times n$. 

The key insight: attention matrix $A \in \mathbb{R}^{K \times K}$ has operator norm $\leq 1$ (row-stochastic), same as $n \times n$ case. However, the \emph{effective} Lipschitz scales with the ``density'' of interactions. For selected tokens, the second pass sees only $K$ other tokens, giving effective Lipschitz $(K/n) \cdot L_{\text{lip}}^2 = \rho L_{\text{lip}}^2$.

\textbf{Combined:}
\begin{equation}
\text{Lip}_{\text{total}} = (1-\rho) \cdot L_{\text{lip}} + \rho \cdot (\rho L_{\text{lip}}^2) = L_{\text{lip}}(1-\rho + \rho^2 L_{\text{lip}}). 
\end{equation}
When $\rho < 1/L_{\text{lip}}$, we have $1-\rho + \rho^2 L_{\text{lip}} < 1$, so modified layers are more stable than standard layers.

\end{proof}

\begin{lemma}[Rademacher Complexity Bound]
\label{lem:rademacher}
\begin{equation}
\mathfrak{R}_N(\mathcal{F}_{\text{topk}}) \leq \frac{C R_{\ln} L_{\text{lip}}^L \sqrt{W_{\text{topk}}\log W_{\text{topk}}}}{\sqrt{N}} \cdot \exp\left(-\frac{c_0 K^2 L}{n^2}\right) + \tilde{\mathcal{O}}\left(\sqrt{\frac{dL}{N}}\right).
\end{equation}
\end{lemma}

\begin{proof}[Detailed Proof]
\textbf{Step 1: Lipschitz product across layers.}

For $M = L/2$ modified layers and $L-M$ standard layers:
\begin{equation}
\prod_{l=1}^L \text{Lip}^{(l)} \leq L_{\text{lip}}^{L-M} \cdot \prod_{l \in \mathcal{L}_{\text{mod}}} L_{\text{lip}}(1-\rho+\rho^2 L_{\text{lip}}) = L_{\text{lip}}^L \exp\left(-\frac{M\rho(1-\rho L_{\text{lip}})}{2}\right),
\end{equation}
using $1-x \leq e^{-x}$ and $M = L/2$.

\textbf{Step 2: Dudley's entropy integral.}

The covering number satisfies $\mathcal{N}(\mathcal{F}_{\text{topk}}, \epsilon) \leq (C'W_{\text{topk}}/\epsilon)^{W_{\text{topk}}}$. Applying Dudley's theorem:
\begin{equation}
\mathfrak{R}_N \leq \inf_{\alpha>0}\left\{4\alpha + \frac{12}{\sqrt{N}}\int_{\alpha}^{\text{diam}} \sqrt{\log \mathcal{N}(\epsilon)} d\epsilon\right\}.
\end{equation}

The Lipschitz reduction enters as $\exp(-c_0 K^2 L/n^2)$ in the final bound. The $\tilde{\mathcal{O}}(\sqrt{dL/N})$ term accounts for the routing network's Rademacher complexity. The $\exp(-c_0 K^2 L/n^2)$ factor is crucial: more sparsity (larger $K^2/n^2$) and deeper networks (larger $L$) exponentially reduce complexity.
\end{proof}

\subsubsection{Stage III: Putting It All Together}

\begin{proof}[Proof of Theorem~\ref{thm:main}]
We combine training error and generalization via the standard decomposition:
\begin{equation}
\mathcal{L}_{\mathcal{D}}(f) = \underbrace{\hat{\mathcal{L}}_S(f)}_{\text{training error}} + \underbrace{(\mathcal{L}_{\mathcal{D}}(f) - \hat{\mathcal{L}}_S(f))}_{\text{generalization gap}}.
\end{equation}

\textbf{Step 1: Apply Proposition~\ref{prop:training} for training error.}
\begin{equation}
\hat{\mathcal{L}}_S(f_{\text{topk}}) \leq \hat{\mathcal{L}}_S(f_{\text{std}}) - \frac{\delta_0 K L}{4n}\eta_{\text{eff}} + \text{error}_1.
\end{equation}

\textbf{Step 2: Bound generalization gap via Lemma~\ref{lem:rademacher}.}
With probability $\geq 1-\delta$:
\begin{equation}
|\mathcal{L}_{\mathcal{D}}(f_{\text{topk}}) - \hat{\mathcal{L}}_S(f_{\text{topk}})| \leq 2\mathfrak{R}_N(\mathcal{F}_{\text{topk}}) + \mathcal{O}\left(\sqrt{\frac{\log(1/\delta)}{N}}\right).
\end{equation}

\textbf{Step 3: Compare to standard Transformer.}

The standard Transformer has:
\begin{equation}
\mathfrak{R}_N(\mathcal{F}_{\text{std}}) \geq \frac{C' R_{\ln} L_{\text{lip}}^L \sqrt{W_{\text{std}}\log W_{\text{std}}}}{\sqrt{N}},
\end{equation}
without the exponential reduction factor. Since $W_{\text{topk}} = W_{\text{std}}(1+O(1/d))$, the complexity advantage is:
\begin{equation}
\mathfrak{R}_N(\mathcal{F}_{\text{std}}) - \mathfrak{R}_N(\mathcal{F}_{\text{topk}}) \geq \Omega\left(\frac{K^2 L}{n^2}\sqrt{\frac{W_{\text{std}}}{N}}\right).
\end{equation}

\textbf{Step 4: Final combination.}

Writing $\mathcal{L}_{\mathcal{D}}(f_{\text{std}}) = \hat{\mathcal{L}}_S(f_{\text{std}}) + \text{Gap}_{\text{std}}$ and similarly for $f_{\text{topk}}$:
\begin{align}
\mathcal{L}_{\mathcal{D}}(f_{\text{topk}}) &= \hat{\mathcal{L}}_S(f_{\text{topk}}) + \text{Gap}_{\text{topk}} \\
&\leq \hat{\mathcal{L}}_S(f_{\text{std}}) - \frac{\delta_0 K L}{4n}\eta_{\text{eff}} + \text{Gap}_{\text{std}} - (\text{Gap}_{\text{std}} - \text{Gap}_{\text{topk}}) + \text{error terms} \\
&= \mathcal{L}_{\mathcal{D}}(f_{\text{std}}) - \frac{\delta_0 K L}{4n}\eta_{\text{eff}} - \Omega\left(\frac{K^2 L}{n^2}\sqrt{\frac{W_{\text{std}}}{N}}\right) + \tilde{\mathcal{O}}\left(\sqrt{\frac{W_{\text{topk}}\log(N/\delta)}{N}}\right).
\end{align}

This completes the proof.
\end{proof}

\subsubsection{Interpretation and Discussion - When Does Token-wise Depth Scaling Win?}
Theorem~\ref{thm:main} shows improvement when:
\begin{equation}
\frac{\delta_0 K L}{4n}\eta_{\text{eff}} + \Omega\left(\frac{K^2 L}{n^2}\sqrt{\frac{W}{N}}\right) > \tilde{\mathcal{O}}\left(\sqrt{\frac{W_{\text{topk}}}{N}}\right).
\end{equation}

This holds when:
\begin{itemize}
    \item $\eta_{\text{eff}}$ is bounded away from zero (hard tokens exist and are learnable)
    \item $N$ is large enough (routing learns well)
    \item $K = \Theta(n/\sqrt{L})$ balances adaptive gain vs. sparsity benefit (i.e., only important tokens perform interaction within depth scaling)
\end{itemize}


\section{More Experimental Details}
\label{app:experiment}
\paragraph{Training Dataset and Evaluation Benchmark.}
To facilitate reproducibility, we first detail the training data configuration. During the supervised fine-tuning (SFT) stage, we curate a subset of approximately 30K samples from OneThinker~\cite{DBLP:journals/corr/abs-2512-03043}, selected for its diverse distribution of reasoning chain lengths. The statistical distributions of CoT length, reasoning steps, and topic categories within this subset are illustrated in Fig.~\ref{fig:data_distribution}. We evaluate our method on three tasks across twelve benchmarks: (1) Mathematics Reasoning (MathVista~\cite{mathvista}, MathVision~\cite{mathvision}, MM-Math~\cite{mmmath}); (2) Vision-centric Reasoning (Hallusion-Bench~\cite{hallusionbench}, MMVP~\cite{mmvp}, Seed-Bench-2-Plus~\cite{seedbench2plus}, HR-Bench~\cite{hrbench}, GQA~\cite{gqa}); (3) Multimodal Composition Reasoning (MMStar~\cite{mmstar}, BLINK~\cite{blink}, ScienceQA~\cite{lu2022learn}, M$^{3}$CoT~\cite{chen-etal-2024-m3cot}).

\paragraph{Implementation Details.} 
To comprehensively evaluate the effectiveness and scalability of our approach, we instantiate our method across a diverse set of backbone architectures, including Qwen2-VL-2B/7B~\citep{wang2024qwen2}, Qwen2.5-VL-3B/7B~\citep{bai2025qwen25vl}, and Chameleon-7B~\citep{team2024chameleon}.

All frameworks employ eager attention mode to enable explicit access to internal attention maps. We set the number of latent reasoning tokens to $T_r = 4$, with $\alpha = 32$ salient regions injected per refinement iteration. To balance computational efficiency and reasoning accuracy, we cap the maximum refinement 
depth at $D = 1$. For the cosine annealing schedule, the region injection count decays from $\alpha_s = 64$ to $\alpha_e = 16$. Models are trained for 16 epochs 
by default.

We employ DeepSpeed ZeRO-2 optimization without CPU offloading, with a per-GPU batch size of 8. Training utilizes the Adam optimizer with a learning rate of $4 \times 10^{-5}$ and $\beta_1 = 0.9$. For the proposed self-distillation loss, we set the weighting coefficient to $\lambda = 1.0$ for 2B/3B-scale models and $\lambda = 0.2$ for 7B-scale models, respectively. Input images are resized to 512 $\times$ 512 resolution when training on OneThinker subset and GQA, and resized to 280 $\times$ 280 on M$^{3}$CoT and ScienceQA. All experiments are conducted on 16 NVIDIA H20 GPUs (96GB VRAM each). A single training run takes approximately 20 hours.

\begin{figure*}[htbp]  
    \centering
    \begin{subfigure}[b]{0.4\linewidth}
        \centering
        \includegraphics[width=\linewidth, height=3cm]{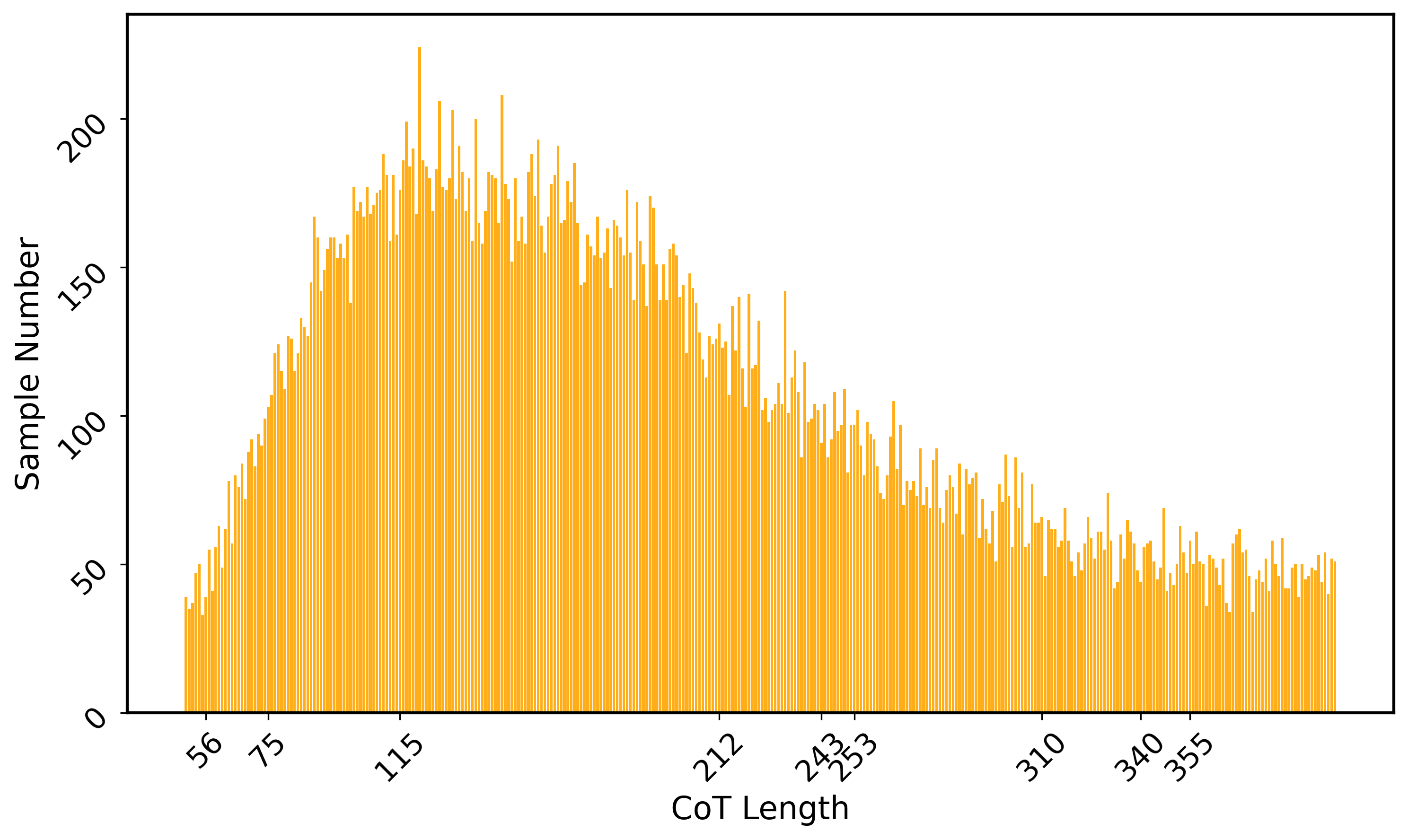}
        \caption{CoT Length Distribution}
        \label{fig:cot_len}
    \end{subfigure}
    \hfill
    \begin{subfigure}[b]{0.32\linewidth} 
        \centering
        \includegraphics[width=\linewidth, height=3cm]{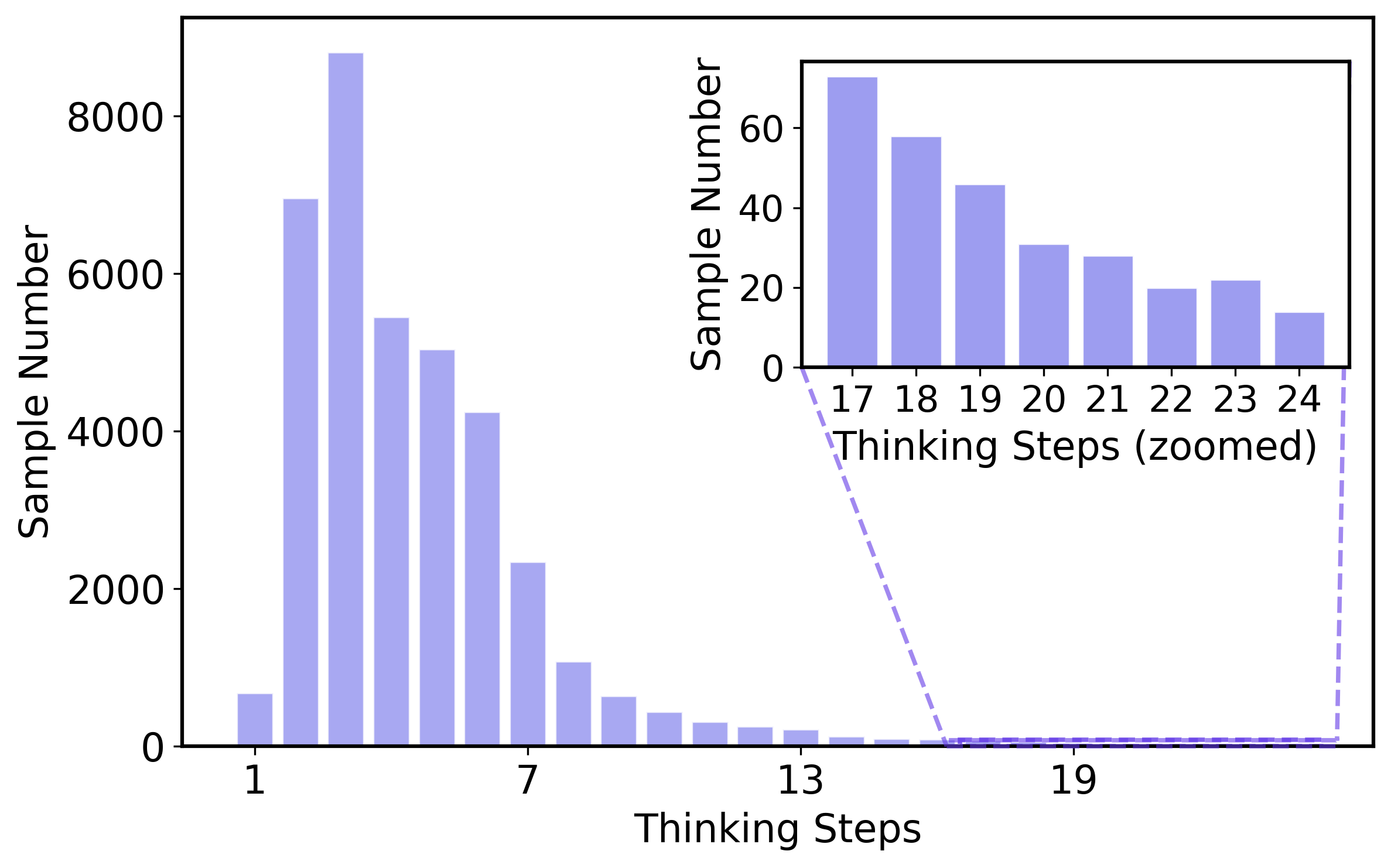}
        \caption{Thinking Step Distribution}
        \label{fig:step_len}
    \end{subfigure}%
    \hfill
    \begin{subfigure}[b]{0.26\linewidth}  
        \centering
        \includegraphics[width=\linewidth, height=3.3cm]{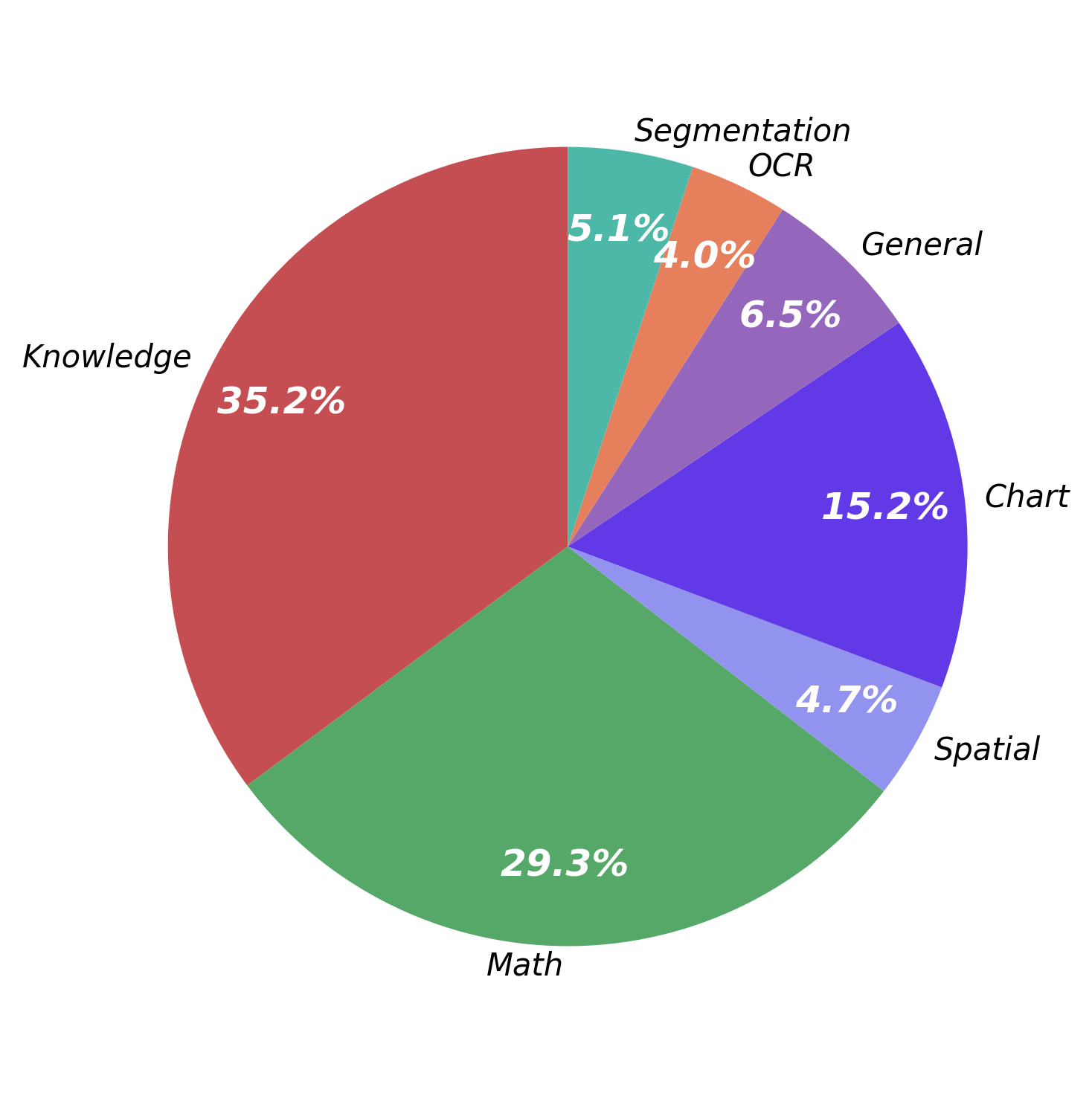}
        \caption{Topic Distribution}
        \label{fig:topic}
    \end{subfigure}%
    \caption{Details of sampled training distribution. Panel (a) depicts the sample distribution across different CoT lengths. The distribution exhibits a boundary at 125 tokens, with sample counts gradually decreasing toward both shorter and longer extremes. Panel (b) illustrates the distribution of reasoning steps in the original data. Specifically, we segment reasoning steps using $\texttt{\textbackslash n\textbackslash n}$ as delimiters; during training, we retain up to 4 delimiters to evenly partition sequence length, while samples containing fewer than 4 delimiters preserve their original segmentation. Panel (c) summarizes the topic-wise distribution of samples in the dataset.
    }
    \label{fig:data_distribution}
\end{figure*}

\section{Baseline Details}
\label{app:baseline}
We evaluate our method against a comprehensive set of state-of-the-art baselines, which can be categorized into four paradigms: (1)
Zero-shot VLMs (GPT-4o~\cite{gpt4v}, LLaVA-OneVision~\cite{li2024llava}, InternVL3.5-8B~\cite{wang2025internvl35}, Qwen2.5-VL-7B~\cite{bai2025qwen25vl}), (2) Explicit CoT-based methods (SCAFFOLD~\citep{lei2024scaffolding}, ICoT~\citep{gao2024interleaved}, Multimodal-CoT~\citep{zhang2023multimodal}, CCoT~\citep{MitraCCoT}, Chain-of-Focus~\citep{zhang2025chain}), (3) Visual Enhanced methods, including tool-augmented reasoning (DeepEyes~\cite{zheng2025deepeyes}) and RL-enhanced VLM reasoning (Vision-R1~\cite{huang2025vision}, PAPO~\cite{wang2025perception}, VL-Rethinker~\cite{wang2025vl}), and (4) Multimodal Latent Reasoning approaches (Laser~\cite{laser2026forest}, LVR~\cite{li2025lvr}, Monet~\cite{wang2025monet}, DMRL~\cite{liu2025reasoningminddynamicmultimodal}).

\vspace{0.5em}
\noindent\textbf{(1) Zero-shot VLMs}

\begin{itemize}
\item \textbf{GPT-4o}~\cite{gpt4v}. GPT-4o is a proprietary, end-to-end multimodal model capable of natively processing text, audio, and visual inputs. It provides robust real-time multimodal reasoning and state-of-the-art instruction-following performance across diverse vision-language tasks.

\item \textbf{LLaVA-OneVision}~\cite{li2024llava}. LLaVA-OneVision is a unified open-source multimodal model designed for versatile visual understanding, including single-image, multi-image, and video inputs. It exhibits superior cross-scenario transferability, effectively bridging the gap between static image analysis and dynamic video comprehension.

\item \textbf{InternVL3.5-8B}~\cite{wang2025internvl35}. InternVL3.5 represents a series of high-performance open-source multimodal models that enhance reasoning capabilities and computational efficiency. Through the integration of cascade reinforcement learning and dynamic visual-resolution routing, this model achieves strong performance. We utilize the 8B variant as a representative open-source baseline.

\item \textbf{Qwen2.5-VL-7B}~\cite{bai2025qwen25vl}. Qwen2.5-VL is a flagship open-source VLM series optimized for advanced recognition, spatial localization, and document/video understanding. It employs dynamic resolution processing and absolute time encoding to facilitate long-video comprehension. We adopt the 7B variant as a key open-source baseline for our comparisons.

\end{itemize}

\vspace{0.5em}
\noindent\textbf{(2) Explicit CoT-based Methods}
\begin{itemize}
    \item \textbf{SCAFFOLD}~\citep{lei2024scaffolding}. SCAFFOLD proposes a prompting strategy designed to enhance vision-language coordination through coordinate scaffolding; Specifically, it superimposes a dot matrix onto the image to establish visual anchors, while leveraging multi-dimensional coordinates as explicit textual positional references.
    \item \textbf{ICoT}~\citep{gao2024interleaved}. ICoT introduces an Interleaved-modal Chain-of-Thought framework that derives answers through sequential steps of paired visual and textual rationales. Additionally, an attention-driven selection mechanism enables adaptation to existing VLMs by intelligently inserting informative image regions guided solely by the model's attention maps.
    \item \textbf{Multimodal-CoT}~\citep{zhang2023multimodal}. Multimodal-CoT incorporates language (text) and vision (images) modalities into a two-stage framework that separates rationale generation and answer inference.
    \item \textbf{CCoT}~\citep{MitraCCoT}. Compositional Chain-of-Thought (CCoT) is a zero-shot prompting method designed to extract compositional knowledge from LMMs through scene graph (SG) representations. By first deriving an SG from the LMM and incorporating it into the reasoning chain, the proposed method effectively enhances the compositional accuracy of the generated responses.
    \item \textbf{Chain-of-Focus}~\citep{zhang2025chain}. Chain-of-Focus (CoF) empowers VLMs to perform adaptive spatial zooming on critical image regions. By leveraging visual cues aligned with the given query, CoF enables dynamic, region-aware reasoning, significantly enhancing the efficiency and precision of multimodal understanding.
\end{itemize}

\noindent\textbf{(3) Tool-use \& RL Enhanced Reasoning}
\begin{itemize}

\item \textbf{DeepEyes}~\cite{zheng2025deepeyes}. To facilitate ``thinking with images'', DeepEyes adopts an interleaved reasoning paradigm. By optimizing tool-assisted visual interaction through end-to-end reinforcement learning, it ensures robust visual grounding and effectively suppresses hallucinations, which are further mitigated by carefully designed data selection and reward shaping strategies.

\item \textbf{Vision-R1}~\cite{huang2025vision}. Extending the R1-style training paradigm to the multimodal domain, Vision-R1 establishes a large-scale chain-of-thought (CoT) dataset to bootstrap MLLM reasoning. The framework subsequently employs reinforcement learning, incorporating techniques like progressive thinking suppression to prune suboptimal reasoning paths and elevate overall multimodal problem-solving accuracy.

\item \textbf{PAPO}~\cite{wang2025perception}. Recognizing the perception-action discrepancy in multimodal RLVR, PAPO (Perception-Aware Policy Optimization) introduces an implicit perception loss defined by a KL-divergence term. This architectural innovation allows integration into standard GRPO/DAPO frameworks, enabling performance gains in vision-dependent tasks without necessitating auxiliary reward or teacher models.

\item \textbf{VL-Rethinker}~\cite{wang2025vl}. VL-Rethinker advocates for a ``slow-thinking'' approach in VLMs by embedding explicit self-reflection cycles into the reasoning process. By augmenting the GRPO objective with selective sample replay, the model is compelled to critically evaluate its intermediate logic, thereby demonstrating significant improvements in complex, multi-step visual reasoning.
\end{itemize}

\vspace{0.5em}
\noindent\textbf{(4) Multimodal Latent Reasoning}
\begin{itemize}
    \item \textbf{Laser}~\cite{laser2026forest}. Laser introduces dynamic windowed alignment learning to reformulate visual deduction. Rather than relying on point-wise predictions, it aligns latent states with temporal windows of future semantics. This enforces a cognitive hierarchy that prioritizes global feature superposition before converging on local details.
    \item \textbf{LVR}~\cite{li2025latent}. Latent Visual Reasoning (LVR) enables autoregressive reasoning in the visual embedding space by training the model to generate latent states that reconstruct key visual tokens, interleaved with text generation; it can be further combined with RL to balance latent reasoning and textual outputs.
    \item \textbf{Monet}~\cite{wang2025monet}. Monet enables latent visual reasoning by generating continuous visual embeddings as intermediate ``visual thoughts'' and proposes a multi-stage distillation SFT pipeline plus visual-latent policy optimization to better train reasoning in latent visual space.
    \item \textbf{DMRL}~\cite{liu2025reasoningminddynamicmultimodal}. DMLR presents a test-time framework that enables dynamic multimodal reasoning. By employing confidence-guided policy gradient optimization, DMLR refines latent thought tokens to enhance the depth and precision of the model's internal reasoning process.
\end{itemize}

\section{Benchmark Details}
\label{app:benchmark}

\noindent \textbf{(1) Visual Perception Benchmarks}
\begin{itemize}
\item \textbf{MMVP}~\cite{mmvp} identifies the inherent limitations of CLIP-based vision encoders by employing 150 pairs of visual patterns that are perceptually distinct to humans yet challenging for latent embeddings. It serves as a diagnostic tool to expose models that provide overconfident but ungrounded predictions, effectively decoupling superficial pattern matching from true visual comprehension.
\item \textbf{SEEDBench2Plus}~\cite{seedbench2plus} focuses on complex, text-intensive visual environments such as charts, maps, and web screenshots. By utilizing human-curated, multiple-choice diagnostics, it rigorously assesses an LVLM’s proficiency in integrated reading and logical reasoning within high-density multimodal contexts.
\item \textbf{HallusionBench}~\cite{hallusionbench} acts as a stress test for the entanglement between linguistic hallucinations and visual illusions. It employs a controlled, question-based framework to systematically isolate and quantify logical inconsistency, offering a granular analysis of common failure modes in LVLM perception.
\item \textbf{HRBench}~\cite{hrbench} probes the limitations of traditional downsampling in high-resolution visual tasks. By leveraging 4K/8K imagery, it evaluates whether models can preserve and interpret native-resolution information, identifying failures where critical visual nuances are lost to aggressive resizing.
\end{itemize}

\noindent \textbf{(2) Compositional Reasoning Benchmarks}
\begin{itemize}
\item \textbf{BLINK}~\cite{blink} emphasizes perception-centric reasoning by posing 3.8K questions that demand immediate, high-fidelity visual engagement—ranging from depth and geometry perception to multi-view correspondence. It is specifically engineered to defeat language-prior dependencies, ensuring that reasoning is strictly grounded in the provided imagery.
\item \textbf{MMStar}~\cite{mmstar} provides a diagnostic assessment of both perception and reasoning through a purified, vision-indispensable dataset. By systematically mitigating language-only shortcuts and data contamination, it offers a robust, multi-axis framework for evaluating the depth of a model's multimodal integration.
\end{itemize}

\noindent \textbf{(3) Mathematical Reasoning Benchmarks}
\begin{itemize}
\item \textbf{MMVista}~\cite{mathvista} is a comprehensive multimodal benchmark specifically designed to evaluate a model's ability to reason over complex visual-textual information. It features a diverse set of tasks that require the integration of visual perception and logical reasoning, often involving high-resolution imagery and multi-step problem solving in real-world scenarios.
\item \textbf{MM-Math}~\cite{mmmath} is a specialized benchmark targeting the intersection of multimodal perception and mathematical reasoning. It comprises a collection of geometry, algebra, and calculus problems that are visually grounded, requiring models to accurately extract quantitative data from diagrams, graphs, and tables to perform precise symbolic calculations.
\item \textbf{MathVision}~\cite{mathvision} is a rigorous benchmark designed to evaluate the mathematical reasoning capabilities of large multimodal models (LMMs) in visually-rich contexts. It contains thousands of challenging, human-annotated problems spanning various mathematical disciplines. Unlike standard benchmarks, MathVision mandates that models interpret complex visual structures—such as functional plots and geometric figures—to derive a solution, effectively testing the model's proficiency in visual-to-math translation.
\end{itemize}

\section{Parametric Analysis}
\label{app:para}

\paragraph{Visual Latents Formation.} 
The \textit{Soft-Mix} strategy aggregates 32 selected patches into a single visual latent token via weighted summation, guided by reweighted attention scores. In contrast, the 
32-patch and 16-patch strategies directly utilize the top 32 and 16 patches with the highest attention scores as visual latents at each reasoning step, respectively. As summarized in Table~\ref{tab:ablation_param}, the 32-patch strategy yields superior performance. Notably, \textit{Soft-Mix} achieves performance comparable to 16-patch while utilizing only a single compact representation. We attribute this performance to the expanded search scope provided by the 32-patch strategy, which offers the router a richer pool of visual tokens, thereby effectively raising the upper bound of depth scaling capabilities.

\begin{wrapfigure}[9]{r}{0.4\linewidth}
\begin{center}
    \includegraphics[width=\linewidth]{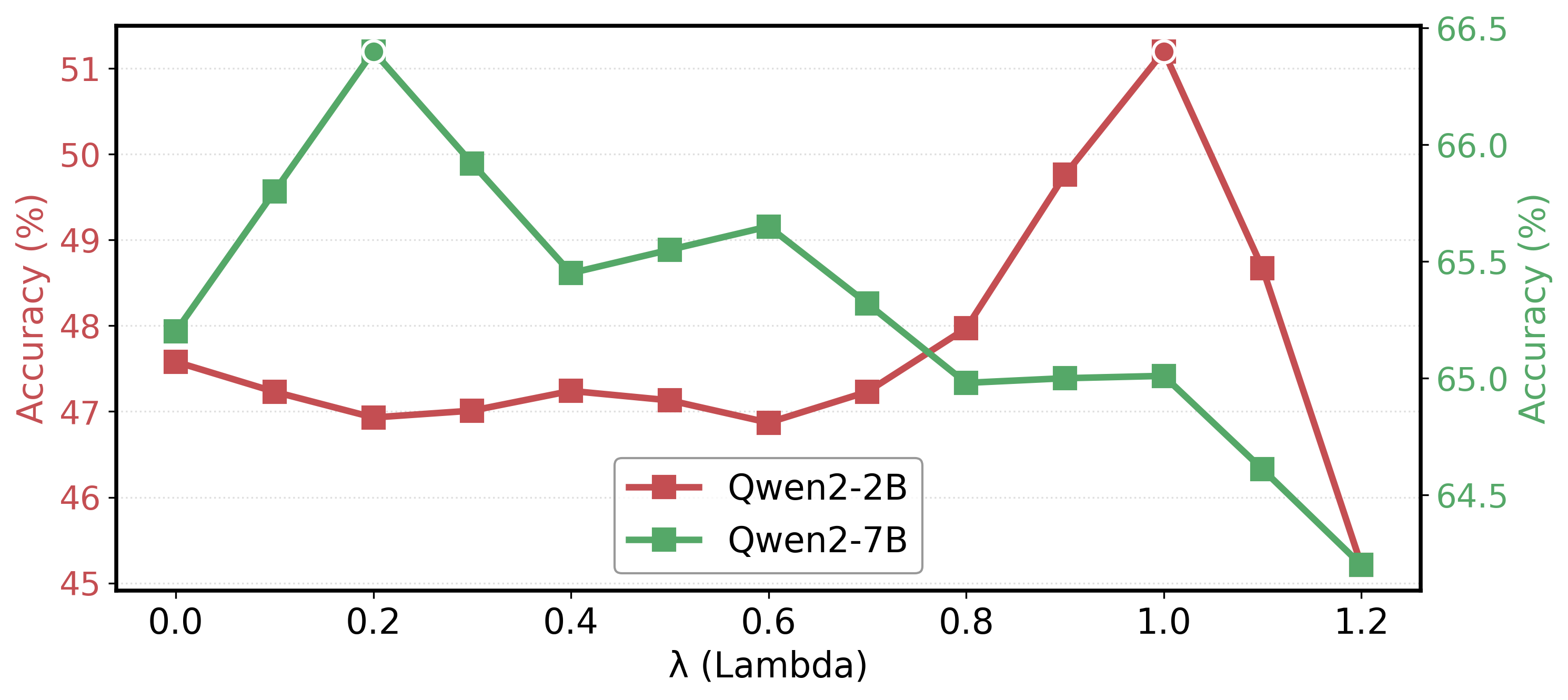}
    \end{center}
    \vspace{-1em}
    \caption{Sensitivity analysis of $\lambda$.}
    \label{fig:lambda}
\end{wrapfigure}

\paragraph{Impact of $\lambda$.} 
As shown in Fig.~\ref{fig:lambda}, model performance exhibits high sensitivity to variations in the hyperparameter $\lambda$. Notably, larger models (e.g., 7B) achieve peak performance at $\lambda=1.0$, whereas the 2B model attains optimal results at a substantially lower value of $\lambda=0.2$. We attribute this divergence to distinct capacity across model scales: smaller models are primarily bottlenecked by visual perception capabilities, thus requiring larger degree of visual play to establish effective visual grounding for downstream reasoning. In contrast, larger models possess sufficient grounding capability, and they benefit from a more balanced allocation between visual and linguistic signals, avoiding over-reliance on visual features that could otherwise suppress the development of complex reasoning chains.

\begin{wraptable}[14]{r}{0.4\textwidth}
  \centering
  \renewcommand{\arraystretch}{1.1}
  \caption{Performance comparison of different sub-grid window sizes on the \textbf{M$^3$CoT} benchmark. Here, $W$ denotes the sub-grid window size; $W=10$ is the window size of extracted visual features. We set $\lambda=1.0$ for the 2B/3B models and $\lambda=0.2$ for the 7B model.}
  \label{tab:window_size}
  \scalebox{0.6}{
      \begin{tabular}{llccc}
        \toprule[1.5pt]
        \multicolumn{2}{c}{} & \multicolumn{3}{c}{\textbf{M$^3$CoT} (Acc.\%)} \\
        \cmidrule(lr){3-5}
        \textbf{Method} & \textbf{Grid Size} & $W=2$ & $W=3$ & $W=5$ \\
        \midrule
        Qwen2-VL-2B   & \multirow{4}{*}{\centering\footnotesize 10$\times$10} & 50.8 & 51.2 & 50.3 \\
        Qwen2-VL-7B   &  & 65.7 & 66.4 & 65.2 \\
        Qwen2.5-VL-3B &  & 65.0 & 65.3 & 64.5 \\
        Qwen2.5-VL-7B &  & 73.4 & 73.9 & 72.1 \\
        \bottomrule[1.5pt]
      \end{tabular}
  }
\end{wraptable}

\paragraph{Impact of different window size $W$.} As shown in Table~\ref{tab:window_size}, $W$ = 3 achieves optimal performance. We reckon that smaller window sizes may fail to comprehensively cover salient visual regions, whereas larger window sizes, while providing broader contextual information, tend to introduce excessive visual noise.

\paragraph{Effect of Retention Parameter $\alpha$.} For the parameter $\alpha$ in $T_{\alpha}(\cdot)$, we compare different retention strategies. As illustrated in Table~\ref{tab:ablation_param}, we evaluate fixed top-$\alpha$ settings with $\alpha \in \{16, 32, 64\}$ and a \textbf{cosine-annealed token retention schedule} (CTR), which gradually reduces the token retention ratio layer by layer following a cosine schedule. Formally, for a model with $L$ layers, the retention ratio for the $l$-th layer is,
\begin{equation}
\alpha(l) =  \mathrm{round} \left( \alpha_{s} + \frac{\alpha_{s} - \alpha_{e}}{2} \left[ 1 + \cos \left( \frac{\pi l}{L} \right) \right] \right),
\end{equation}
where $\alpha_{s}$ and $\alpha_{e}$ are two endpoints that enables flexible control over computational cost. As can be seen panel (b) in Table~\ref{tab:ablation_param}, results show that fixed top-$\alpha$ with $\alpha=32$ achieves optimal performance, while cosine-annealed schedules yield consistently inferior results. We attribute this result to that: as visual tokens and latent hidden states are gradually appended to the token sequence, larger contextual scopes becomes essential for deeper token interaction. Consequently, strategies that progressively decrease the token retention ratio during depth scaling unavoidably limit these critical interactions, thereby hindering the model's capacity to develop comprehensive contextual understanding.



\end{document}